\documentclass[11pt]{article}
\usepackage[thmeqnumber]{matus}
\setlist[itemize]{leftmargin=2em}
\setlist[enumerate]{leftmargin=2em}

\def\mc{|\hspace{-0.21em}\rangle}
\def\sr{\textup{sr}}

\def\tX{\tilde X}
\def\tW{\tilde W}
\def\tcO{{\widetilde{\cO}}}
\def\phig{\phi_\gamma}
\def\Lip{\textup{Lip}}
\def\mnist{\texttt{mnist}\xspace}
\def\cifar{\texttt{cifar10}\xspace}

\def\skde{p_{n,\beta}}

\def\tb{\tilde{b}}

\tcbset{size=fbox,boxrule=0pt,bottomrule=0pt,sharp corners = all}

\title{\textbf{Generalization bounds via distillation}}

\author{
Daniel Hsu\thanks{\texttt{<\href{mailto:djhsu@cs.columbia.edu}{djhsu@cs.columbia.edu}>};
Columbia University, New York City.}
\and
Ziwei Ji\thanks{\texttt{<\{\href{mailto:ziweiji2@illinois.edu}{ziweiji2},\href{mailto:mjt@illinois.edu}{mjt},\href{mailto:lanwang2@illinois.edu}{lanwang2}\}@illinois.edu>};
University of Illinois, Urbana-Champaign.}
\and
Matus Telgarsky\footnotemark[2]
\and
Lan Wang\footnotemark[2]
}

\date{} 

\begin{document}

\maketitle

\begin{abstract}
  This paper theoretically investigates the following empirical phenomenon:
  given a high-complexity network with poor generalization bounds, one can \emph{distill}
  it into a network with nearly identical predictions but low complexity and vastly smaller 
  generalization bounds.
  The main contribution is an analysis showing that the original network
  inherits this good generalization bound from its distillation, assuming the use of well-behaved data augmentation.
  This bound is presented both in an abstract and in a concrete form, the latter complemented
  by a reduction technique to handle modern computation graphs featuring convolutional layers, fully-connected
  layers, and skip connections, to name a few.
  To round out the story, a (looser) classical uniform convergence analysis of compression is also presented,
  as well as a variety of experiments on \cifar and \mnist demonstrating
  similar generalization performance
  between the original network and its distillation.
\end{abstract}

\section{Overview and main results}

Generalization bounds are statistical tools which take as input various measurements
of a predictor on training data, and output a performance estimate for
unseen data --- that is, they estimate how well the predictor \emph{generalizes} to unseen data.
Despite extensive development spanning many decades \citep{anthony_bartlett_nn},
there is growing concern that these bounds are not only disastrously loose
\citep{roy_vacuous}, but worse that they do not correlate with the underlying
phenomena \citep{fantastic_measures},
and even that the basic method of proof is doomed
\citep{rethinking,nagarajan_kolter__uniform_convergence}.
As an explicit demonstration of the looseness of these bounds, \Cref{fig:bounds}
calculates bounds for a standard ResNet architecture
achieving test errors of respectively 0.008 and 0.067 on \mnist and \cifar;
the observed generalization gap is $10^{-1}$,
while standard generalization techniques upper bound it with $10^{15}$.

Contrary to this dilemma, there is evidence that these networks can often be
compressed or \emph{distilled} into simpler networks, while still preserving their output values
and low test error.  Meanwhile, these simpler networks exhibit vastly better generalization
bounds: again referring to \Cref{fig:bounds}, those same networks from before can be distilled 
with hardly any change to their outputs, while their bounds reduce by a factor of roughly $10^{10}$.
Distillation is widely studied \citep{caruana_distillation,hinton_distillation},
but usually the
original network is discarded and only the final distilled network is preserved.

The purpose of this work is to carry the good generalization bounds of the distilled network
back to the original network;
in a sense, the explicit simplicity of the distilled network is used
as a witness to implicit simplicity of the original network.
The main contributions are as follows.

\begin{itemize}
  \item
    The main theoretical contribution is a generalization bound for the original,
    undistilled network which scales primarily with the generalization properties of
    its distillation, assuming that well-behaved data augmentation is used to measure the
    \emph{distillation distance}.
    An abstract version of this bound is stated in \Cref{fact:main:fix:3}, along with a
    sufficient data augmentation technique in \Cref{fact:augmentation}.
    A concrete version of the bound, suitable to handle the ResNet architecture in
    \Cref{fig:bounds}, is described in \Cref{fact:pollard:snb}.  Handling sophisticated
    architectures with only minor proof alterations is another contribution of this work,
    and is described alongside \Cref{fact:pollard:snb}.  This abstract and concrete analysis is
    sketched in \Cref{sec:analysis}, with full proofs deferred to appendices.

  \item
    Rather than using an assumption on the distillation process
    (e.g., the aforementioned ``well-behaved data augmentation''),
    this work also gives a direct uniform convergence analysis, culminating
    in \Cref{fact:rad:sr}.
    This is presented partially as an open problem or cautionary tale, as its proof is vastly
    more sophisticated than that of \Cref{fact:pollard:snb}, but ultimately
    results in a much looser analysis.  This analysis is sketched in \Cref{sec:analysis},
    with full proofs deferred to appendices.

  \item
    While this work is primarily theoretical, it is motivated by \Cref{fig:bounds} and
    related experiments: \Cref{fig:width,fig:srb,fig:random} demonstrate that not only
    does distillation improve generalization upper bounds, but moreover it makes them sufficiently
    tight to capture intrinsic properties of the predictors, for example removing the usual
    bad dependence on width in these bounds (cf. \Cref{fig:width}).
    These experiments are detailed in \Cref{sec:exp}.

\end{itemize}

\begin{figure}[t]
  \centering
  \begin{subfigure}[b]{0.590\textwidth}
    \includegraphics[width = 1.0\textwidth]{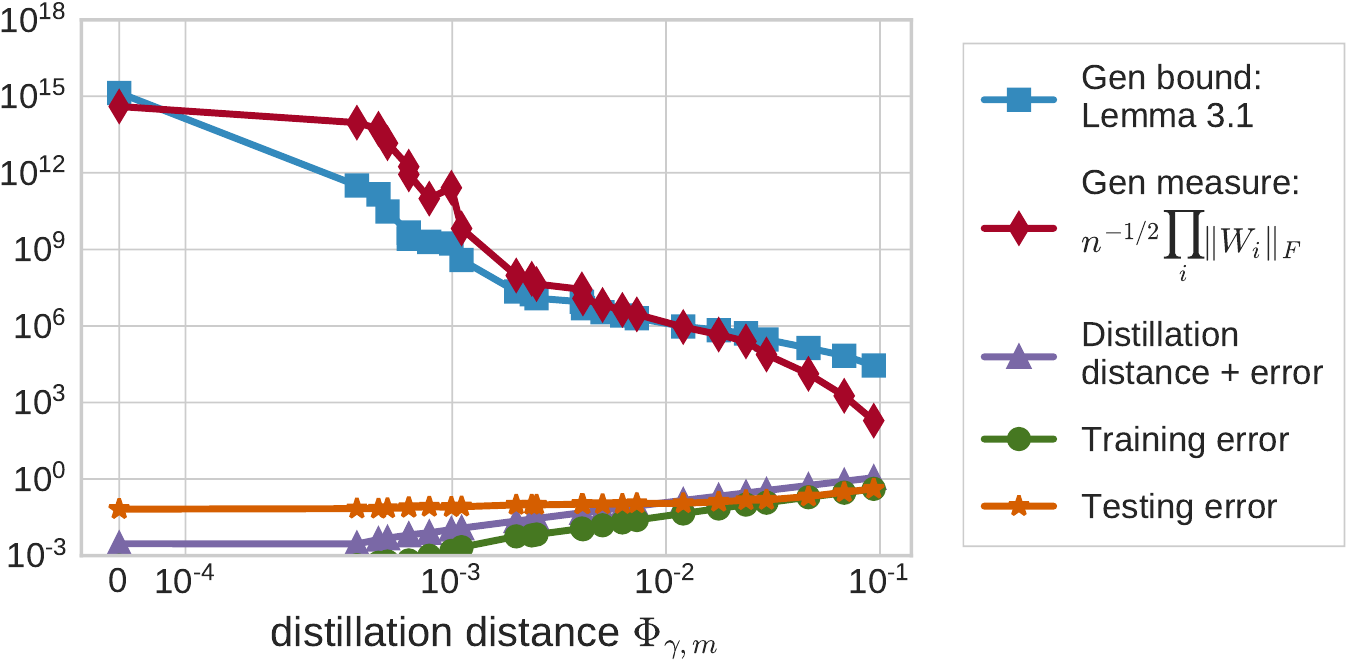}
    \caption{ResNet8 trained on \cifar.\hspace{6em}\phantom{lol}}
  \end{subfigure}
\begin{subfigure}[b]{0.401\textwidth}
    \includegraphics[width = 1.0\textwidth]{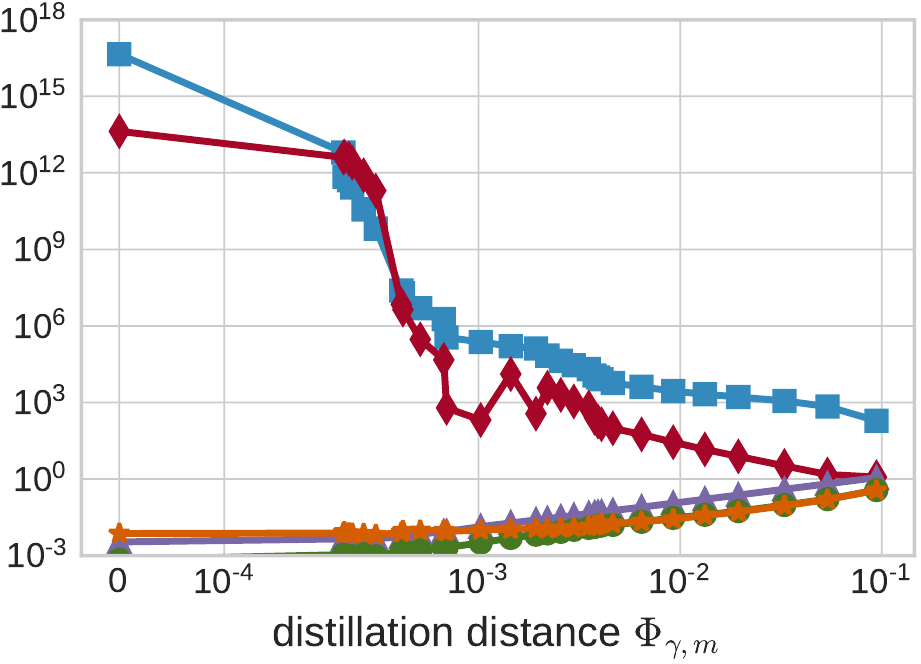}
    \caption{ResNet8 trained on \mnist.}
  \end{subfigure}
\caption{\textbf{Generalization bounds throughout distillation.}
    These two subfigures track a sequence of increasingly distilled/compressed ResNet8
    networks along
    their horizontal axes, respectively for \cifar and \mnist data.
    This horizontal axis measures \emph{distillation distance $\Phi_{\gamma,m}$},
    as defined below in \cref{eq:Phi}.
    The bottom curves measure various training and testing errors, whereas the top two
    curves measure respectively a generalization bound presented here (cf.
  \Cref{fact:pollard:snb} and \Cref{fact:snb}), and a \emph{generalization measure}.
Notably, the top two curves drop throughout a long interval
    during which test error
    remains small.
  For further experimental details, see \Cref{sec:exp}.}
  \label{fig:bounds}
\end{figure}

\subsection{An abstract bound via data augmentation}\label{sec:abstract}

This subsection describes the basic distillation setup
and the core abstract bound based on data augmentation,
culminating in \Cref{fact:main:fix:3,fact:augmentation};
a concrete bound follows in \Cref{sec:concrete}.

Given a multi-class predictor
$f:\R^d\to\R^k$, distillation finds another predictor $g:\R^d\to\R^k$ which is simpler,
but close in \emph{distillation distance $\Phi_{\gamma,m}$},
meaning the \emph{softmax outputs $\phig$}
are close on average over a set of points $(z_i)_{i=1}^m$:
\begin{equation}
\Phi_{\gamma,m}(f,g) := \frac 1 m \sum_{i=1}^m\enVert{ \phig(f(z_i)) - \phig(g(z_i))}_1,
\quad \textup{where }
\phig(f(z)) \propto \exp\del{f(z)/\gamma}.
\label{eq:Phi}
\end{equation}
The quantity $\gamma > 0$ is sometimes called a \emph{temperature} \citep{hinton_distillation}.
Decreasing $\gamma$ increases sensitivity
near the decision boundary; in this way, it is naturally related to the concept of \emph{margins}
in generalization theory, as detailed in \Cref{sec:margins}.
due to these connections, the use of softmax is beneficial in this work,
though not completely standard in the literature \citep{caruana_distillation}.

We can now outline \Cref{fig:bounds} and the associated
empirical phenomenon which motivates this work.  (Please see \Cref{sec:exp} for further
details on these experiments.)
Consider a predictor $f$ which has good test error but bad generalization bounds;
by treating the distillation distance $\Phi_{\gamma,m}(f,g)$ as an objective function
and increasingly regularizing $g$,
we obtain a sequence of predictors $(g_0,\ldots,g_t)$, where $g_0=f$, which trade off
between distillation distance and predictor complexity.
The curves in \Cref{fig:bounds} are produced in exactly this way, and
demonstrate that there are predictors nearly identical to the original $f$
which have vastly smaller generalization bounds.

Our goal here is to show that this is enough to imply that $f$ in turn must \emph{also} have
good generalization bounds, despite its apparent complexity.
To sketch the idea,
by a bit of algebra (cf. \Cref{fact:phig}),
we can upper bound error probabilities with \emph{expected} distillation distances and errors:
\[
  \Pr_{x,y}[\argmax_{y'}f(x)_{y'} \neq y]
\leq
2 \bbE_{x} \enVert{ \phig(f(x)) - \phig(g(x)) }_1
  + 2 \bbE_{x,y} \del{ 1 - \phig(g(x))_y}.
\]
The next step is to convert these expected errors into quantities over the training set.
The last term is already in a form we want: it depends only on $g$,
so we can apply uniform convergence with the low complexity of $g$.
(Measured over the training set, this term is the \emph{distillation error}
in \Cref{fig:bounds}.)

The expected distillation distance term is problematic, however.  Here are two approaches.
\begin{enumerate}
  \item
    We can directly apply uniform convergence; for instance, this approach was followed by \citet{suzuki2019compression},
    and a more direct approach is followed here to prove \Cref{fact:rad:sr}.
    Unfortunately, it is unclear how this technique can avoid paying
    significantly for the high complexity of $f$.

  \item
    The idea in this subsection is to somehow trade off computation for the high statistical cost
    of the complexity of $f$.  Specifically, notice that $\Phi_{\gamma,m}(f,g)$
    only relies upon the marginal distribution
    of the inputs $x$, and not their labels.  This subsection will pay computation to estimate
    $\Phi_{\gamma,m}$ with extra samples via \emph{data augmentation},
    offsetting the high complexity of $f$.
\end{enumerate}

We can now set up and state our main distillation bound.  Suppose we have a training set 
$((x_i,y_i))_{i=1}^n$ drawn from some measure $\mu$, with marginal distribution $\mu_\cX$ on
the inputs $x$.
Suppose we also have $(z_i)_{i=1}^m$ drawn from a \emph{data augmentation distribution}
$\nu_n$, the subscript referring to the fact that it depends on $(x_i)_{i=1}^n$.  Our
analysis works when $\|\nicefrac{\dif\mu_{\cX}}{\dif\nu_n}\|_\infty$, the ratio between the two densities,
is finite.  If it is large, then one can tighten the bound by sampling more from $\nu_n$,
which is a computational burden;
explicit bounds on this term will be given shortly in \Cref{fact:augmentation}.

\begin{lemma}
  \label{fact:main:fix:3}
  Let temperature parameter $\gamma > 0$ be given,
  along with sets of multiclass predictors $\cF$ and $\cG$.
  Then with probability at least $1-2\delta$ over an iid draw of data $((x_i,y_i))_{i=1}^n$ from
  $\mu$ and $(z_i)_{i=1}^m$ from $\nu_n$,
  every $f\in\cF$ and $g\in\cG$ satisfy
  \begin{align*}
    \Pr[\argmax_{y'} f(x)_{y'} \neq y]
    &\leq 2 
      \enVert{ \frac {\dif \mu_{\cX}}{\dif \nu_n} }_{\infty}
      \Phi_{\gamma,m}(f, g)
      +
      \frac 2 n \sum_{i=1}^n \del{1 - \phig(g(x_i))_{y_i}}
      \\
    &
      +
      \tcO\del[3]{
        \frac {k^{3/2}}{\gamma}
      \enVert{ \frac {\dif \mu_{\cX}}{\dif \nu_n} }_{\infty}
      \del{ \Rad_m(\cF) + \Rad_m(\cG)}
    + \frac {\sqrt k}{\gamma}\Rad_n(\cG) }
      \\
      &
      +
    6 \sqrt{\frac{\ln(1/\delta)}{2n}}
      \del{ 1 + 
          \enVert{ \frac {\dif \mu_{\cX}}{\dif \nu_n} }_{\infty}
          \sqrt{\frac n m }
      },
  \end{align*}
  where Rademacher complexities $\Rad_n$ and $\Rad_m$ are defined in \Cref{sec:notation}.
\end{lemma}

A key point is that the Rademacher complexity $\Rad_m(\cF)$ of the complicated
functions $\cF$ has a subscript ``$m$'', which explicitly introduces a factor $\nicefrac 1 m$
in the complexity definition (cf. \Cref{sec:notation}).  As such, sampling more from the 
data augmentation measure can mitigate this term, and leave the complexity of the distillation
class $\cG$ as the dominant term.

Of course, this also requires $\|\nicefrac{\dif\mu_{\cX}}{\dif\nu_n}\|_\infty$ to be reasonable.
As follows is one data augmentation scheme (and assumption on marginal distribution $\mu_\cX$)
which ensures this.

\begin{lemma}
  \label{fact:augmentation}
  Let $(x_i)_{i=1}^n$ be a data sample drawn iid from $\mu_\cX$, and suppose
  the corresponding density $p$ is supported on $[0,1]^d$ and is H{\"o}lder continuous,
  meaning $|p(x) - p(x')| \leq C_\alpha \|x-x'\|^\alpha$ for some $C_\alpha\geq 0,\alpha\in [0,1]$.
  Define a data augmentation measure $\nu_n$ via the following sampling procedure.
  \begin{itemize}
    \item
      With probability $1/2$, sample $z$ uniformly within $[0,1]^d$.
    \item
      Otherwise, select a data index $i\in[n]$ uniformly,
      and sample $z$ from a Gaussian centered at $x_i$, and having covariance
      $\sigma^2 I$ where $\sigma := n^{-1 / (2\alpha + d)}$.
  \end{itemize}
  Then with probability at least $1-1/n$ over the draw of $(x_i)_{i=1}^n$,
  \[
    \enVert{\frac{\dif\mu_{\cX}}{\dif\nu_n}}_\infty
    =
    4+
    \cO\del{
      \frac {\sqrt{\ln n}}{n^{\alpha / (2\alpha + d)}}
    }.
  \]
\end{lemma}

Though the idea is not pursued here, there are other ways to control
$\|\nicefrac{\dif\mu_{\cX}}{\dif\nu_n}\|_\infty$, for instance via an independent
sample of unlabeled data;
\Cref{fact:main:fix:3} is agnostic to these choices.

\subsection{A concrete bound for computation graphs}
\label{sec:concrete}

This subsection gives an explicit complexity bound which starts from \Cref{fact:main:fix:3},
but bounds $\|\nicefrac{\dif\mu_{\cX}}{\dif\nu_n}\|_\infty$ via \Cref{fact:augmentation},
and also includes an upper bound on Rademacher complexity which can handle the ResNet, as in
\Cref{fig:bounds}.   A side contribution of this work is the formalism to easily handle
these architectures, detailed as follows.

\emph{Canonical computation graphs} are a way to write down feedforward networks
which include dense linear layers, convolutional
layers, skip connections, and multivariate gates, to name a few, all while allowing the analysis
to look roughly like a regular dense network.
The construction applies directly to batches:
given an input batch $X\in\R^{n\times d}$, the output $X_i$ of layer $i$ is defined
inductively as 
\[
  X_0^\T := X^\T,
  \qquad
  X_i^\T := \sigma_i\del{[ W_i \Pi_i D_i \mc F_i ] X_{i-1}^\T}
  =
  \sigma_i\del{\sbr{
  \begin{smallmatrix} W_i \Pi_i D_i X_{i-1}^\T \\ F_i  X_{i-1}^\T
  \end{smallmatrix}}},
\]
where:
$\sigma_i$ is a multivariate-to-multivariate $\rho_i$-Lipschitz function (measured over
minibatches on either side with Frobenius norm);
$F_i$ is a \emph{fixed} matrix, for instance
an identity mapping as in a residual network's skip connection;
$D_i$ is a \emph{fixed} diagonal matrix selecting certain coordinates, for instance the non-skip
part in a residual network; $\Pi_i$ is a Frobenius norm projection of a full minibatch;
$W_i$ is a weight matrix, the trainable parameters;
$[ W_i \Pi_i D_i \mc F_i ]$ denotes \emph{row-wise concatenation}
of $W_i\Pi_i D_i$ and $F_i$.

As a simple example of this architecture, a multi-layer skip connection can be modeled
by including identity mappings in all relevant fixed matrices $F_i$, and also including identity
mappings in the corresponding coordinates of the multivariate gates $\sigma_i$.
As a second example, note how to model convolution layers:
each layer outputs a matrix whose rows correspond to examples, but nothing prevents the batch
size from changes between layers;
in particular, the multivariate activation before a convolution layer can reshape its output
to have each row correspond to a patch of an input image, whereby the convolution filter is
now a regular dense weight matrix.

A fixed computation graph architecture $\cG(\vec \rho,\vec b,\vec r, \vec s)$
has associated hyperparameters
$(\vec \rho, \vec b, \vec r, \vec s)$, described as follows.
$\vec\rho$ is the set of Lipschitz constants for each (multivariate) gate, as described before.
$r_i$ is a norm bound $\|W_i^\T\|_{2,1}\leq r_i$ (sum of the $\|\cdot\|_2$-norms of the rows),
$b_i\sqrt{n}$ (where $n$ is the input batch size)
is the radius of the Frobenius norm ball which $\Pi_i$ is projecting onto,
and $s_i$ is the operator norm of $X \mapsto [ W_i \Pi_i D_i X^\T \mc F_i X^\T ]$.
While the definition is intricate, it cannot only model basic residual networks,
but it is sensitive enough to be able to have $s_i=1$ and $r_i=0$ when residual blocks are fully
zeroed out, an effect which indeed occurs during distillation.

\begin{theorem}
  \label{fact:pollard:snb}
  Let temperature parameter $\gamma > 0$ be given,
  along with multiclass predictors $\cF$, and a computation graph architecture
  $\cG$.
  Then with probability at least $1-2\delta$ over an iid draw of data $((x_i,y_i))_{i=1}^n$ from
  $\mu$ and $(z_i)_{i=1}^n$ from $\nu_n$,
  every $f\in\cF$ satisfies
  \begin{align*}
    \Pr[\argmax_{y'} f(x)_{y'} \neq y]
    &\leq
    \inf_{\substack{(\vec b, \vec r, \vec s)\geq 1\\g\in\cG(\vec \rho,\vec b, \vec r, \vec s)}}
    2 
    \Bigg[
      \enVert{ \frac {\dif \mu_{\cX}}{\dif \nu_n} }_{\infty}
      \Phi_{\gamma,m}(f, g)
      +
      \frac 2 n \sum_{i=1}^n (1 - \phig(g(x_i))_{y_i}
      \\
    &
      +
      \tcO\del[3]{
        \frac {k^{3/2}}{\gamma}
        \enVert{ \frac {\dif \mu_{\cX}}{\dif \nu_n} }_{\infty}
      \Rad_m(\cF) }
      +
      6 \sqrt{\frac{\ln(1/\delta)}{2n}}
      \del{ 1 + 
          \enVert{ \frac {\dif \mu_{\cX}}{\dif \nu_n} }_{\infty}
          \sqrt{\frac n m }
      }
      \\
    &\hspace{-1em}+
      \tcO\del[4]{
        \frac {\sqrt{k}}{\gamma \sqrt n}
        \del{1 + k\enVert{ \frac {\dif \mu_{\cX}}{\dif \nu_n} }_{\infty} \sqrt{\frac{n}{m}}}
      \del[4]{ \sum_i \sbr[3]{r_i b_i\rho_i \prod_{l=i+1}^L s_l\rho_l}^{2/3} }^{3/2}
      }
    \Bigg].
  \end{align*}
  Under the conditions of \Cref{fact:augmentation},
  ignoring an additional failure probability $\nicefrac 1 n$,
  then
  $\|\frac{\dif\mu_{\cX}}{\dif\nu_n}\|_\infty = 4 + \cO\del{
      \frac {\sqrt{\ln n}}{n^{\alpha / (2\alpha + d)}}
    }$.
\end{theorem}

A proof sketch of this bound appears in \Cref{sec:analysis}, with full details deferred to
appendices.  The proof is a simplification of the covering number argument
from \citep{spec};
for another computation graph formalism designed to work with the covering number arguments
from \citep{spec}, see the generalization bounds due to \citet{wei2019datadependent}.

\begin{figure}[t]
  \centering
  \begin{subfigure}[b]{0.465\textwidth}
\includegraphics[width = 1.0\textwidth]{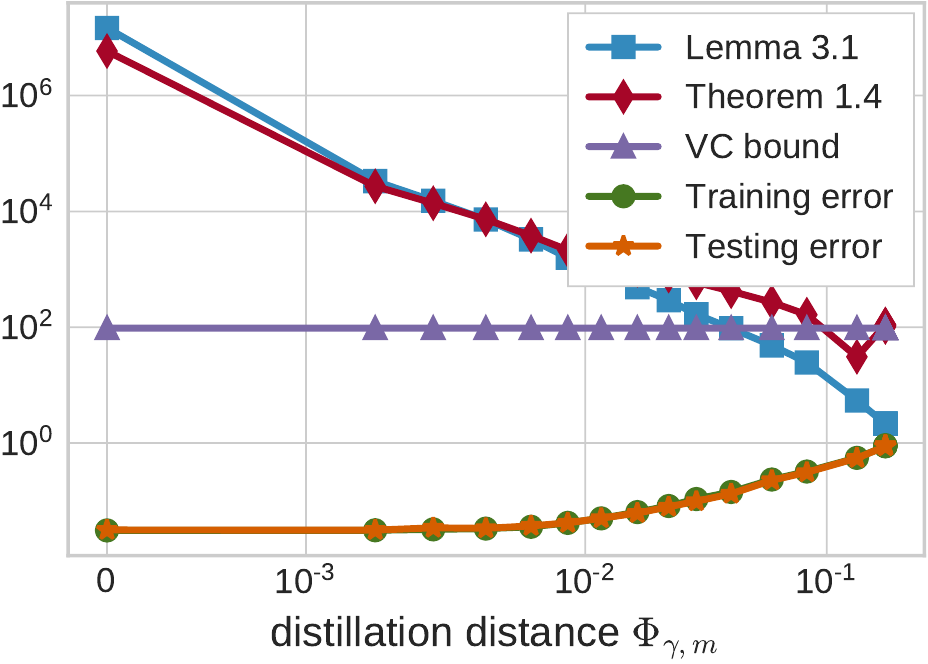}
    \caption{Comparison of bounds on \cifar.}
    \label{fig:srb:a}
  \end{subfigure}
  ~\quad~
  \begin{subfigure}[b]{0.465\textwidth}
    \includegraphics[width = 1.0\textwidth]{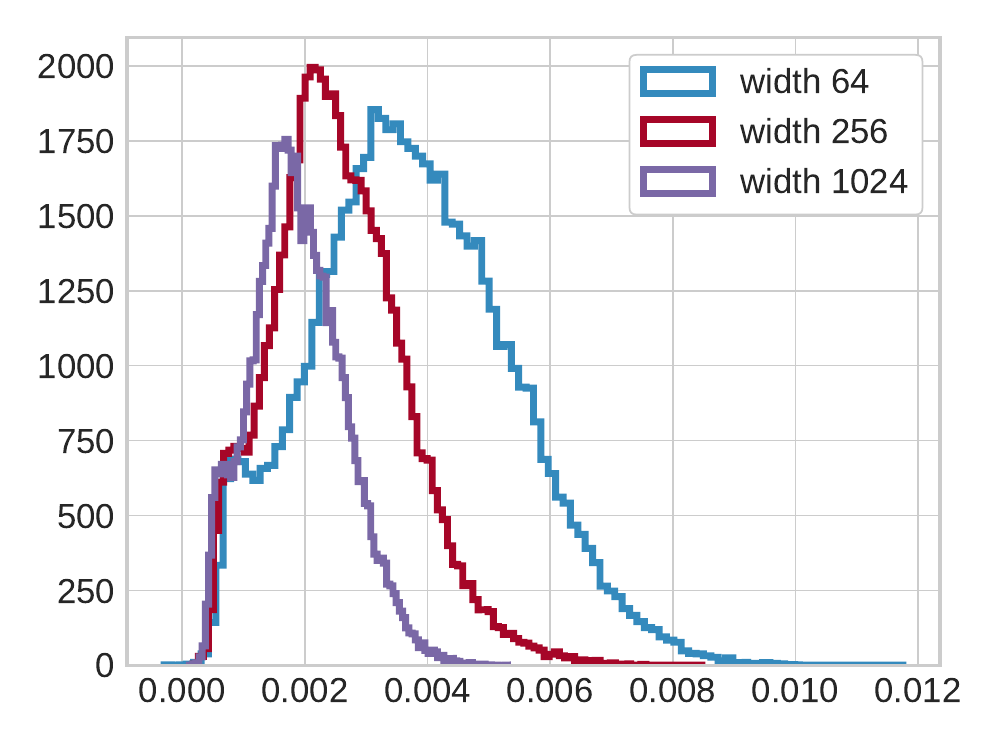}
    \caption{Width dependence with \Cref{fact:rad:sr}.}
    \label{fig:srb:b}
  \end{subfigure}
\caption{\textbf{Performance of stable rank bound (cf. \Cref{fact:rad:sr}).}
    \Cref{fig:srb:a} compares \Cref{fact:rad:sr} to \Cref{fact:snb} and the VC bound
    \citep{relu_vc}, and \Cref{fig:srb:b} normalizes the \emph{margin histogram}
      by \Cref{fact:rad:sr}, showing an unfortunate
    failure of width independence (cf.
  \Cref{fig:width}).  For details and a discussion of margin histograms, see \Cref{sec:exp}.}
  \label{fig:srb}
\end{figure}

\subsection{A uniform-convergence approach to distillation}\label{sec:srb}

In this section, we derive a Rademacher complexity bound on $\cF$ whose proof \emph{internally}
uses compression; specifically, it first replaces $f$ with a narrower network $g$, and then
uses a covering number bound sensitive to network size to control $g$.
The proof analytically chooses $g$'s width based on the
structure of $f$ and also the provided data, and this data dependence incurs a factor which 
causes the familiar $\nicefrac1 {\sqrt n}$ rate to worsen to $\nicefrac 1 {n^{1/4}}$
(which appears as $\nicefrac {\|X\|_{\tF}}{n^{3/4}}$).
This proof is much more intricate than the proofs coming before, and cannot handle general
computation graphs, and also ignores the beneficial structure of the softmax.

\begin{theorem}
  \label{fact:rad:sr}
  Let data matrix $X\in\R^{n\times d}$ be given,
  and
  let $\cF$ denote networks of the form $x \mapsto \sigma_L(W_L\cdots \sigma_1(W_1 x))$
  with spectral norm $\|W_i\|_2 \leq s_i$, and $1$-Lipschitz and $1$-homogeneous
  activations $\sigma_i$,
  and $\|W_i\|_\tF \leq R_i$
  and width at most $m$.  Then
  \[
    \Rad(\cF) = \tcO\del[3]{
      \frac {\|X\|_\tF} {n^{3/4}} 
      \sbr[2]{\prod_j s_j}\sbr[2]{\sum_i (R_i/s_i)^{4/5}}^{5/4}
      \sbr[2]{\sum_i \ln R_i}^{1/4}
    }.
  \]
\end{theorem}

The term $\nicefrac {R_i}{s_i}$ is the square root of the \emph{stable rank} of weight matrix
$W_i$, and is a desirable quantity in a generalization bound: it scales more mildly
with width than terms like $\|W_i^\T\|_{2,1}$ and $\|W_i^\T\|_{\tF} \sqrt{\textup{width}}$
which often appear (the former appears in \Cref{fact:pollard:snb} and \Cref{fact:snb}).
Another stable rank bound was developed by \citet{suzuki2019compression}, but has an extra
mild dependence on width.

As depicted in \Cref{fig:srb}, however, this bound is not fully width-independent.  Moreover,
we can compare it to \Cref{fact:snb} throughout distillation, and not only does this bound
not capture the power of distillation, but also, eventually its bad dependence on
$n$ causes it to lose out to \Cref{fact:snb}.

\subsection{Additional notation}
\label{sec:notation}

Given data $(z_i)_{i=1}^n$, the \emph{Rademacher complexity} of univariate functions $\cH$ is
\[
  \Rad(\cH) := \bbE_{\vec \eps} \sup_{h\in\cH} \frac 1 n \sum_i \eps_i h(z_i),
  \qquad\textup{where }
  \eps_i \stackrel{\textup{i.i.d.}}{\sim}\textup{Uniform}(\{-1,+1\}).
\]
Rademacher complexity is the most common tool in generalization theory
\citep{shai_shai_book}, and is incorporated in \Cref{fact:main:fix:3} due to its
convenience and wide use.  To handle multivariate (multiclass) outputs, the definition
is overloaded via the worst case labels as
$\Rad_n(\cF) = \sup_{\vec y\in[k]^n} \Rad(\{(x,y)\mapsto f(x)_y : f\in\cF\})$.
This definition is for mathematical convenience, but overall not ideal;
Rademacher
complexity seems to have difficulty dealing with such geometries.

Regarding norms, $\|\cdot\|=\|\cdot\|_\tF$ will denote the Frobenius norm,
and $\|\cdot\|_2$ will denote spectral norm.

\section{Illustrative empirical results}
\label{sec:exp}

This section describes the experimental setup, and the main experiments:
\Cref{fig:bounds} showing progressive distillation,
\Cref{fig:srb} comparing \Cref{fact:rad:sr}, \Cref{fact:snb} and VC dimension,
\Cref{fig:width} showing width independence after distillation,
and \Cref{fig:random} showing the effect of random labels.

\begin{figure}[t]
  \centering
  \begin{subfigure}[b]{0.465\textwidth}
    \includegraphics[width = 1.0\textwidth]{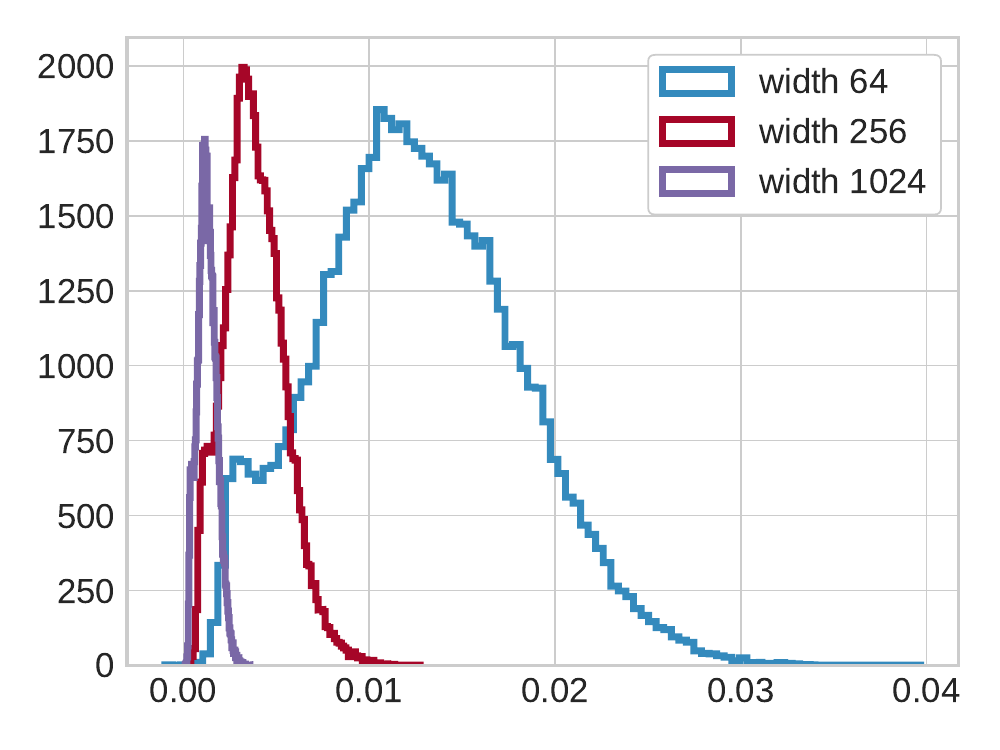}
    \caption{Margins before distillation.\label{fig:width:a}}
  \end{subfigure}
  ~\quad~
  \begin{subfigure}[b]{0.465\textwidth}
    \includegraphics[width = 1.0\textwidth]{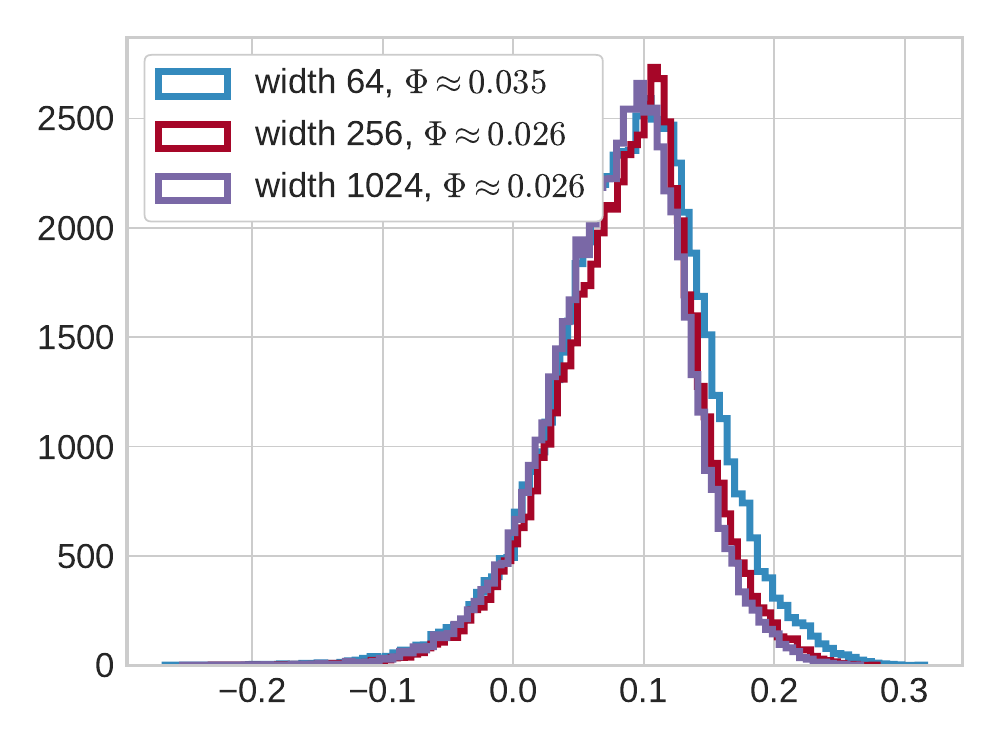}
    \caption{Margins after distillation.\label{fig:width:b}}
  \end{subfigure}
\caption{\textbf{Width independence.} Fully-connected 6-layer networks of widths
    $\{64, 256, 1024\}$ were trained on \mnist until training error zero; the margin histograms,
    normalized by the generalization bound in \Cref{fact:snb}, all differ, and are close to zero.
    After distillation, the margin distributions are far from zero and nearly the same.
    In the distillation legend, the second term $\Phi_{\gamma,m}$
    denotes the distillation distance, as defined in \Cref{eq:Phi}.
  Experiment details and an explanation of margin histograms appear in \Cref{sec:exp}.}
  \label{fig:width}
\end{figure}

\paragraph{Experimental setup.}
As sketched before, networks were trained in a standard way on either \cifar or \mnist,
and then \emph{distilled} by trading off between complexity and distillation distance
$\Phi_{\gamma,m}$.  Details are as follows.

\begin{enumerate}
  \item
    \textbf{Training initial network $f$.}
    In \Cref{fig:bounds,fig:srb:a}, the architecture was a ResNet8 based on one used
    in \citep{dawnbench}, and achieved test errors 0.067 and 0.008 on \cifar and \mnist,
    respectively, with no changes to the setup and a modest amount of training;
    the training algorithm was Adam;
    this and most other choices followed the scheme in \citep{dawnbench}
    to achieve a competitively low test error on \cifar.
    In \Cref{fig:srb:b,fig:width,fig:random}, a 6-layer fully connected network was
    used (width 8192 in \Cref{fig:srb:b}, widths $\{64, 256, 1024\}$ in \Cref{fig:width},
    width 256 in \Cref{fig:random}), and vanilla SGD was used to optimize.

  \item
    \textbf{Training distillation network $g$.}
    Given $f$ and a regularization strength $\lambda_j$, each distillation $g_j$
    was found via approximate minimization of the objective
    \begin{equation}
      g \mapsto \Phi_{\gamma,m}(f,g) + \lambda_j\textup{Complexity(g)}.
      \label{eq:dist_opt}
    \end{equation}
    In more detail,
    first $g_0$ was initialized to $f$ ($g$ and $f$ always used the same architecture)
    and optimized via \cref{eq:dist_opt} with $\lambda_0$ set to roughly
    $\nicefrac{\textup{risk}(f)}{\textup{Complexity}(f)}$, and thereafter $g_{j+1}$ was
    initialized to $g_j$ and found by optimizing \cref{eq:dist_opt} with
    $\lambda_{j+1} := 2\lambda_j$.
    The optimization method was the same as the one used to find $f$.
    The term $\textup{Complexity}(g)$ was some computationally reasonable approximation of
    \Cref{fact:snb}: for \Cref{fig:srb:b,fig:width,fig:random},
    it was just $\sum_i \|W_i^\T\|_{2,1}$, but for \Cref{fig:bounds,fig:srb:a}, it also
    included a tractable surrogate for the product of the spectral norms, which greatly
    helped distillation performance with these deeper architectures.

    In \Cref{fig:srb:b,fig:width,fig:random}, a full regularization sequence was not shown,
    only a single $g_j$.  This was chosen with a simple heuristic: amongst all $(g_j)_{j\geq 1}$,
    pick the one whose 10\% margin quantile is largest
    (see the definition and discussion of margins below).
\end{enumerate}

\begin{figure}[t]
  \centering
  \begin{subfigure}[b]{0.495\textwidth}
\includegraphics[width = 1.0\textwidth]{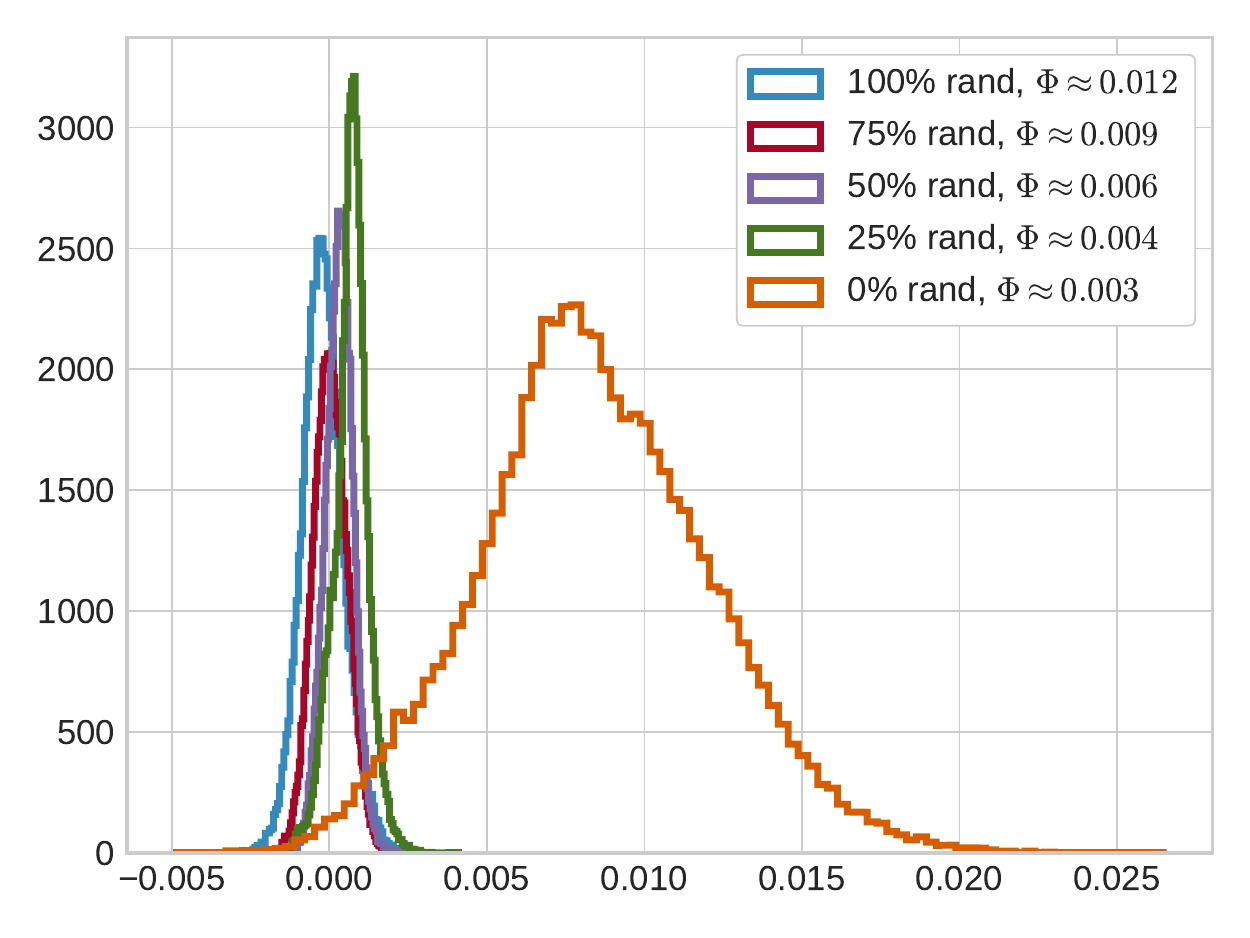}
  \caption{Permuting different fractions of labels.\label{fig:rand:all}}
  \end{subfigure}
  \begin{subfigure}[b]{0.495\textwidth}
\includegraphics[width = 1.0\textwidth]{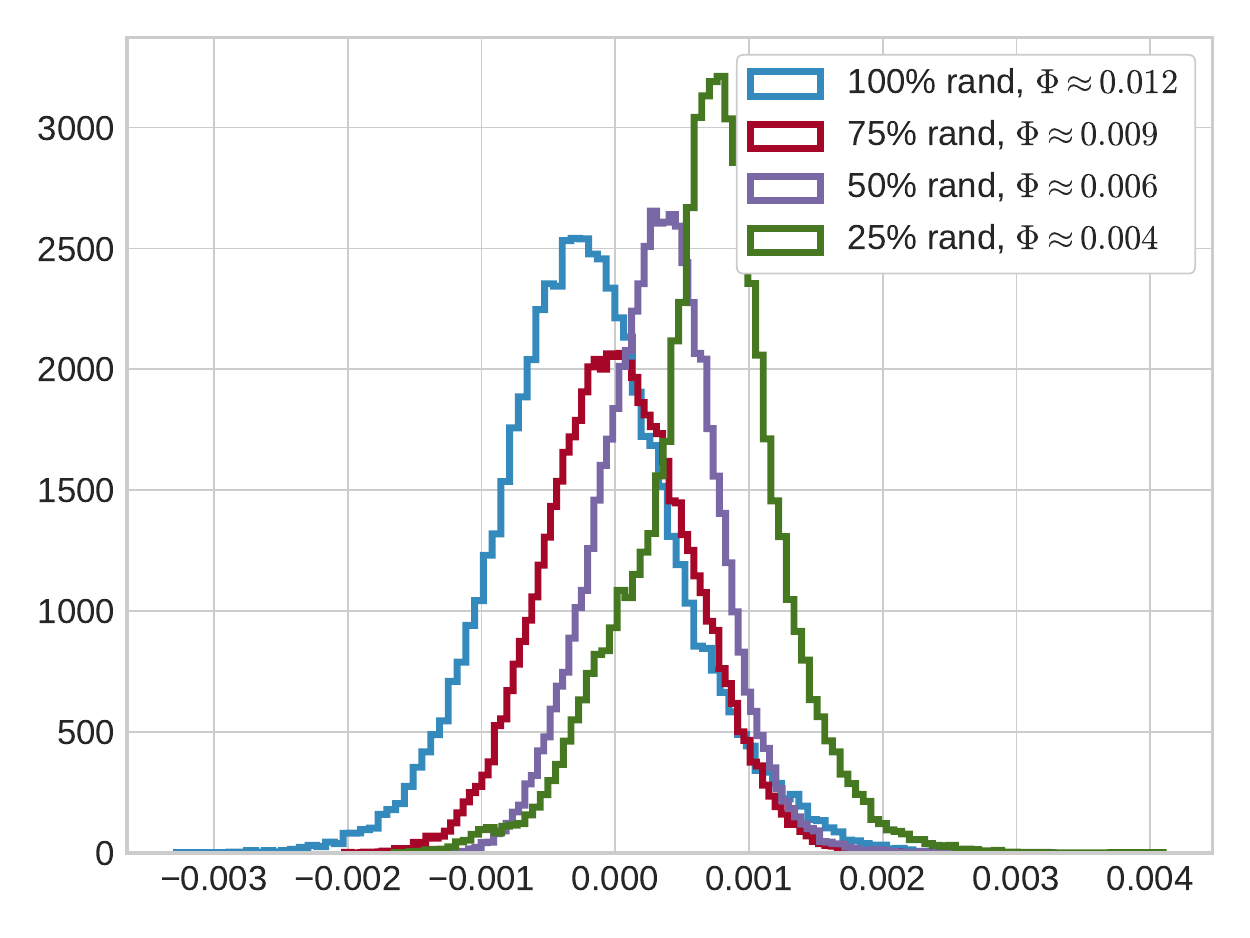}
    \caption{Zoomed in.\label{fig:rand:zoom}}
  \end{subfigure}
\caption{\textbf{Label randomization.}  Here $\{0\%,25\%,50\%,75\%,100\%\}$ of the
    labels were permuted across the respective experiments.  In all cases, the margin
    distribution is collapsed to zero.
For details, including
    an explanation of margin histograms,
  see \Cref{sec:exp}.}
  \label{fig:random}
\end{figure}

\paragraph{Margin histograms.}
\Cref{fig:srb:b,fig:width,fig:random} all depict \emph{margin histograms}, a flexible
tool to study the individual predictions of a network on all examples in a training set
(see for instance \citep{schapire_freund_book_final} for their use studying boosting,
and \citep{spec,margindist_predict} for their use in studying deep networks).
Concretely, given a predictor $g\in\cG$,
the prediction on every example is replaced with a real scalar called the
\emph{normalized margin} via
\[
  (x_i,y_i) \mapsto \frac {g(x_i)_{y_i} - \max_{j\neq y_i} g(x_i)_j}{\Rad_n(\cG)},
\]
where $\Rad_n(\cG)$ is the Rademacher complexity (cf. \Cref{sec:notation}),
and then the histogram of these $n$ scalars is plotted, with the horizontal axis values
thus corresponding to normalized margins.
By using Rademacher complexity as normalization, these margin distributions can be
compared across predictors and even data sets, and give a more fine-grained analysis
of the quality of the generalization bound.
This normalization choice was first studied in \citep{spec}, where it was also mentioned
that this normalization allows one to read off generalization bounds from the plot.
Here, it also suggests reasonable values for the softmax temperature $\gamma$.

\paragraph{\Cref{fig:bounds}: effect of distillation on generalization bounds.}
This figure was described before; briefly, a highlight is that in the initial phase,
training and testing errors
hardly change while bounds drop by a factor of nearly $10^{10}$.
Regarding ``generalization measure'', this term appears in studies of quantities
which correlate with generalization, but are not necessarily rigorous generalization
bounds \citep{fantastic_measures,roy_correlation}; in this specific case, the product of Frobenius
norms requires a dense ReLU network \citep{golowich}, and is invalid for the ResNet
(e.g., a complicated ResNet with a single identity residual block yields a value
$0$ by this measure).

\paragraph{\Cref{fig:srb:a}: comparison of \Cref{fact:rad:sr}, \Cref{fact:snb} and VC bounds.}
\Cref{fact:rad:sr} was intended to internalize distillation, but as in \Cref{fig:srb:a},
clearly a subsequent distillation still greatly reduces the bound.  While initially the 
bound is better than \Cref{fact:snb} (which does not internalize distillation),
eventually the $n^{1/4}$ factor causes it to lose out.  Also note that eventually the
bounds beat the VC bound, which has been identified as a surprisingly challenging baseline
\citep{arora_compression}.

\paragraph{\Cref{fig:width}: width independence.}
Prior work has identified that generalization bounds are quite bad at handling changes in width,
even if predictions and test error don't change much
\citep{nagarajan_kolter__uniform_convergence,fantastic_measures,roy_correlation}.
This is captured in \Cref{fig:width:a}, where the \emph{margin distributions} (see above)
with different widths are all very different, despite similar test errors.
However, following distillation, the margin histograms in \Cref{fig:width:b} are nearly
identical!  That is to say: distillation not only decreases loose upper bounds as before,
it tightens them to the point where they capture intrinsic properties of the predictors.

\paragraph{\Cref{fig:srb:b}: failure of width independence with \Cref{fact:rad:sr}.}
The bound in \Cref{fact:rad:sr} was designed to internalize compression, and there was
some hope of this due to the stable rank term.  Unfortunately, \Cref{fig:srb:b} shows that
it doesn't quite succeed: while the margin histograms are less separated than
for the undistilled networks in \Cref{fig:width:a}, they are still visibly separated
unlike the post-distillation histograms in \Cref{fig:width:b}.

\paragraph{\Cref{fig:random}: random labels.}
A standard sanity check for generalization bounds is whether
they can reflect the difficulty of fitting random labels \citep{rethinking}.  While
it has been empirically shown that Rademacher bounds do sharply reflect the presence of
random labels \citep[Figures 2 \& 3]{spec}, the effect is amplified with distillation:
even randomizing just 25\% shrinks the margin distribution significantly.

\section{Analysis overview and sketch of proofs}
\label{sec:analysis}

This section sketches all proofs, and provides
further context and connections to the literature.  Full proof details appear
in the appendices.

\subsection{Abstract data augmentation bounds in \Cref{sec:abstract}}

As mentioned in \Cref{sec:abstract}, the first step of the proof is to apply
\Cref{fact:phig} to obtain
\[
  \Pr_{x,y}[\argmax_{y'}f(x)_{y'} \neq y]
  \leq
  2 \bbE_{x} \enVert{ \phig(f(x)) - \phig(g(x)) }_1
  + 2 \bbE_{x,y} \del{ 1 - \phig(g(x))_y};
\]
this step is similar to how the ramp loss is used with margin-based generalization bounds,
a connection which is discussed in \Cref{sec:margins}.

\Cref{sec:abstract} also mentioned that the last term is easy: $\phig$ is
$(\nicefrac 1 \gamma)$-Lipschitz, and we can peel it off and only pay the Rademacher complexity
associated with $g\in\cG$.

With data augmentation, the first term is also easy:
\begin{align*}
  \bbE \Phi_{\gamma,m}(f,g)
  &=
  \int \|\phig(f(z)) - \phig(g(z))\|_1 \dif\mu_\cX(z)
  =
  \int \|\phig(f(z)) - \phig(g(z))\|_1 \frac {\dif \mu_\cX}{\dif\nu_n}\dif\nu_n(z)
  \\
  &\leq
  \enVert{\frac {\dif \mu_\cX}{\dif\nu_n}}_\infty \int \|\phig(f(z)) - \phig(g(z))\|_1 \dif\nu_n(z),
\end{align*}
and now we may apply uniform convergence to $\nu_n$ rather than $\mu_{\cX}$.  In the appendix,
this proof is handled with a bit more generality, allowing arbitrary norms, which may help in
certain settings.  All together, this leads to a proof of \Cref{fact:main:fix:3}.

For the explicit data augmentation estimate in \Cref{fact:augmentation}, the proof breaks into
roughly two cases: low density regions where the uniform sampling gives the bound, and high
density regions where the Gaussian sampling gives the bound.  In the latter case, the Gaussian
sampling in expectation behaves as a kernel density estimate, and the proof invokes
a standard bound \citep{jiang-kde}.

\subsection{Concrete data augmentation bounds in \Cref{sec:concrete}}

The main work in this proof is the following generalization bound for computation graphs,
which follows the proof scheme from \citep{spec},
though simplified in various ways, owing mainly to the omission of general matrix norm
penalties on weight matrices, and the omission of the \emph{reference matrices}.  The reference
matrices were a technique to center the weight norm balls away from the origin; a logical
place to center them was at initialization.  However, in this distillation setting, it is in
fact most natural to center everything at the origin, and apply regularization and shrink
to a well-behaved function (rather than shrinking back to the random initialization, which after
all defines a complicated function).
The proof also features a simplified $(2,1)$-norm matrix covering proof (cf. \Cref{fact:21}).

\begin{lemma}
  \label{fact:snb}
  Let data $X\in\R^{n\times d}$ be given.
  Let computation graph $\cG$ be given,
  where $\Pi_i$ projects to Frobenius-norm balls of radius $b_i\sqrt{n}$,
  and $\|W_i^\T\|_{2,1} \leq r_i$,
  and $\|[W_i \Pi_i D_i \mc F_i]\|_2 \leq s_i$,
  and $\Lip(\sigma_i)\leq \rho_i$,
  and all layers have width at most $m$.
  Then for every $\eps>0$ there exists a covering set $\cM$ satisfying
  \[
    \sup_{g\in\cG} \min_{\hat X \in \cM} \enVert{ g(X^\T) - \hat X } \leq \eps
    \quad\textup{and}\quad
    \ln|\cM| \leq
    \frac {2^{4/3} n \ln(2m^2)} {\eps^2}
    \sbr[4]{ \sum_i \del[3]{r_i b_i \rho_i \prod_{l=i+1}^Ls_l\rho_l}^{2/3} }^3.
  \]
  Consequently,
  \[
    \Rad(\cG)
    \leq
    \frac 4 n
    +
    12 \sqrt{\frac {\ln(2m^2)}{n}}
    \sbr[4]{ \sum_i \del[3]{r_i b_i \rho_i \prod_{l=i+1}^L s_l\rho_l}^{2/3} }^{3/2}.
  \]
\end{lemma}

From there, the proof of \Cref{fact:pollard:snb} follows via \Cref{fact:main:fix:3,fact:augmentation},
and many union bounds.

\subsection{Direct uniform convergence approach in \Cref{fact:rad:sr}}

As mentioned before, the first step of the proof is to sparsify the network,
specifically each matrix product.  Concretely, given
weights $W_i$ of layer $i$, letting $X_{i-1}^\T$ denote the input to this layer,
then
\[
  W_i X_{i-1}^\T = \sum_{j=1}^m (W_i\be_j) (X_{i-1}\be_j)^\T.
\]
Written this way, it seems natural that the matrix product should ``concentrate'',
and that considering all $m$ outer products should not be necessary.  Indeed,
exactly such an approach has been followed before to analyze randomized matrix multiplication
schemes \citep{sarlos}. As there is no goal of high probability here,
the analysis is simpler, and follows from the Maurey lemma (cf. \Cref{fact:maurey}),
as is used in the $(2,1)$-norm matrix covering bound in \Cref{fact:21}.

\begin{lemma}
  \label{fact:sparsify:2}
  Let a network be given with $1$-Lipschitz homogeneous activations $\sigma_i$
  and weight matrices $(W_1,\ldots,W_L)$ of maximum width $m$,
  along with data matrix $X\in\R^{n\times d}$
  and desired widths $(k_1,\ldots,k_L)$ be given.
  Then there exists a sparsified network output, recursively defined via
  \[
    \hX_0^\T := X^\T,
    \quad\textup{and}\quad
    \hX_i^\T := \Pi_i \sigma_i(W_i M_i X_{i-1}^\T),
    \quad\textup{where}\quad
    M_i := \sum_{j\in S_i} \frac {Z_j\bfe_j\bfe_j^\T}{\|A\be_j\|},
  \]
  where $S_i$ is a multiset of $k_i = |S_i|$ indices, $\Pi_i$ denotes projection
  onto the Frobenius-norm ball of radius $\|X\|_\tF \prod_{j\leq i}\|W_j\|_2$,
  and the scaling term $Z_j$ satisfies $Z_j \leq \|W_k\|_\tF \sqrt{m/k_j}$,
  and
  \[
    \|\sigma_L(W_L\cdots\sigma_1(W_1X^\T)\cdots) - \hX_L^\T\|_\tF \leq
    \|X\|_\tF
    \sbr{ \prod_{i=1}^L \|W_i\|_2 }
    \sum_{i=1}^L
\sqrt{\frac{\|W_i\|_\tF^2}{k_i\|W_i\|_2^2}},
  \]
\end{lemma}

The statement of this lemma is lengthy and detailed because the exact guts of the
construction are needed in the subsequent generalization proof.  Specifically,
now that there are few nodes, a generalization bound sensitive to narrow networks can be
applied.  On the surface, it seems reasonable to apply a VC bound, but this approach
did not yield a rate better than $n^{-1/6}$, and also had an explicit dependence on the
depth of the network, times other terms visible in \Cref{fact:rad:sr}.

Instead, the approach here, aiming for a better dependence on $n$ and also no explicit
dependence on network depth, was to produce an $\infty$-norm covering number bound
(see \citep{long2019generalization} for a related approach), with
some minor adjustments (indeed, the $\infty$-norm parameter covering approach was applied to
obtain a Frobenius-norm bound, as in \Cref{fact:snb}).  Unfortunately, the magnitudes of weight
matrix entries must be controlled for this to work (unlike the VC approach), and this necessitated
the detailed form of \Cref{fact:sparsify:2} above.

To close with a few pointers to the literature,
as \Cref{fact:sparsify:2} is essentially a pruning bound, it is potentially of
independent interest; see for instance the literature on lottery tickets and pruning
\citep{lottery_ticket,frankle_roy__pruning,jason__pruning}.
Secondly, there is already one generalization bound in the literature which exhibits spectral
norms, due to \citet{suzuki2019compression}; unfortunately, it also has an explicit dependence
on network width.

\subsubsection*{Acknowledgments}
MT thanks Vaishnavh Nagarajan for helpful discussions and suggestions.
ZJ and MT are grateful for support from the NSF under grant IIS-1750051, and
from NVIDIA under a GPU grant.

\bibliographystyle{plainnat}
\bibliography{bib}

\clearpage
\appendix

\section{Proofs for \Cref{sec:abstract}}

The first step is an abstract version of \Cref{fact:main:fix:3} which
does not explicitly involve the softmax, just bounded functions.

\begin{lemma}
  \label{fact:main:fix}
  Let classes of bounded functions $\cF$ and $\cG$ be given
  with $\cF \ni f : \cX \to [0,1]^k$
  and $\cG \ni g : \cX \to [0,1]^k$.
  Let conjugate exponents $1/p + 1/q = 1$ be given.
  Then with probability at least $1-2\delta$ over the draw of $((x_i,y_i))_{i=1}^n$ from $\mu$
  and $(z_i)_{i=1}^m$ from $\nu_n$,
  for every $f\in\cF$ and $g\in\cG$,
  \begin{align*}
    \bbE f(x)_y
    &
    \leq
      \frac 1 n \sum_{i=1}^n g(x_i)_{y_i}
      + 2 \Rad_n\del{\cbr{(x,y) \mapsto g(x)_y : g\in \cG}}
      + 3 \sqrt{\frac{\ln(1/\delta)}{2n}}
      \\
      &
      \quad+
      \enVert{ \frac {\dif \mu_{\cX}}{\dif \nu_n} }_{L_q(\nu_n)}
      \Bigg(
        \frac 1 m \sum_{i=1}^m \|f(z_i) - g(z_i)\|^p_p
        + 3 \sqrt{\frac{\ln(1/\delta)}{2m}}
        \\
      &\qquad
        + 2 \Rad_m\del{\cbr{z\mapsto \min\{1, \|f(z)-g(z)\|_p^p\} : f\in \cF, g\in\cG } }
      \Bigg)^{1/p}
  \end{align*}
  where
  \begin{align*}
    &
    \Rad_m\del{\cbr{z\mapsto \min\{1, \|f(z)-g(z)\|_p^p\} : f\in \cF, g\in\cG } }
    \\
    &\leq
    p \sum_{y'=1}^k\sbr{
      \Rad_m(\{ z\mapsto f(z)_{y'} : f\in \cF\})
    + \Rad_m(\{ z\mapsto g(z)_{y'} : g\in \cG \}) }.
  \end{align*}
\end{lemma}
\begin{proof}[Proof of \Cref{fact:main:fix}]
  To start, for any $f\in\cF$ and $g\in\cG$, write
  \[
    \bbE f(x)_y
    = 
    \bbE (f(x) - g(x))_y + \bbE g(x)_y.
  \]
  The last term is easiest, and let's handle it first: by standard
  Rademacher complexity arguments \citep{shai_shai_book},
  with probability at least $1-\delta$,
  every $g\in\cG$ satisfies
  \[
    \bbE g(x)_y
    \leq
    \frac 1 n \sum_{i=1}^n g(x_i)_{y_i}
    + 2 \Rad_n(\{ (x,y) \mapsto g(x)_y : g\in \cG\})
    + 3 \sqrt{\frac {\ln(1/\delta)}{2n}}.
  \]
  For the first term, since $f:\cX \to [0,1]^k$ and $g:\cX\to[0,1]^k$,
  by H{\"o}lder's inequality
  \begin{align*}
    \bbE (f(x) - g(x))_y
    &= \int \min\{ 1, (f(x)-g(x))_y \} \dif\mu(x,y)
    \\
    &\leq \int \min\{1,  \|f(x) - g(x)\|_p\} \dif\mu(x,y)
    \\
    &= \int  \min\{1,  \|f(x) - g(x)\|_p\} \frac {\dif \mu_\cX}{\dif \nu_n}(x) \dif\nu_n(x)
    \\
    &\leq \enVert{\min\cbr{1,\| f-g\|_p} }_{L_p(\nu_n)}
    \enVert{ \frac {\dif \mu_\cX}{\dif \nu_n} }_{L_q(\nu_n)}.
  \end{align*}
  Once again invoking standard Rademacher complexity arguments \citep{shai_shai_book},
  with probability at least $1-\delta$,
  every mapping  $z\mapsto \min\{1, \|f(z)-g(z)\|_p^p\}$ where $f\in \cF$ and $g\in \cG$ satisfies
  \begin{align*}
    \int \min\{1, \|f(z) - g(z)\|_p^p\} \dif\nu_n(z)
    &\leq
    \frac 1 m \sum_{i=1}^m \min\{1,  \|f(z_i) - g(z_i)\|_p^p\}
    + 3 \sqrt{\frac {\ln(1/\delta)}{2m}}
    \\
    &+ 2\Rad_m\del{\cbr{z\mapsto \min\{1, \|f(z)-g(z)\|_p^p\} : f\in \cF, g\in\cG } }.
  \end{align*}
  Combining these bounds and unioning the two failure events
  gives the first bound.

  For the final Rademacher complexity estimate,
  first note $r\mapsto \min\{1,r\}$ is 1-Lipschitz and can be
  peeled off, thus
  \begin{align*}
    &m\Rad_m\del{\cbr{z\mapsto \min\{1, \|f(z)-g(z)\|_p^p\} : f\in \cF, g\in\cG } }
    \\
    &\leq
    m\Rad_m\del{\cbr{z\mapsto \|f(z)-g(z)\|_p^p : f\in \cF, g\in\cG } }
    \\
    &=
    \bbE_{\eps} \sup_{\substack{f\in\cF\\g\in\cG}} \sum_{i=1}^m \eps_i \|f(z_i)-g(z_i)\|_p^p
    \\
    &\leq
    \sum_{y'=1}^k \bbE_{\eps} \sup_{\substack{f\in\cF\\g\in\cG}} \sum_{i=1}^m \eps_i |f(z_i)-g(z_i)|_{y'}^p
    \\
    &=
    \sum_{y'=1}^k m\Rad_m\del{\cbr{z\mapsto |f(z)-g(z)|_{y'}^p : f\in \cF, g\in\cG } }.
  \end{align*}
  Since $f$ and $g$ have range $[0,1]^k$, then $(f-g)_{y'}$ has range $[-1,1]$ for every $y'$,
  and since $r \mapsto |r|^p$ is $p$-Lipschitz over $[-1,1]$ (for any $p\in [1,\infty)$,
  combining this with the Lipschitz composition rule for Rademacher complexity
  and also the fact that a Rademacher random vector $\eps \in \{\pm 1\}^m$
  is distributionally equivalent to its coordinate-wise negation $-\eps$,
  then, for every $y'\in[k]$,
  \begin{align*}
    &
    \Rad_m(\{z\mapsto |f(z)-g(z)|^p_{y'} : f\in \cF, g\in\cG\})
    \\
    &\leq
    p \Rad_m(\{z\mapsto (f(z)-g(z))_{y'} : f\in \cF, g\in\cG\})
    \\
    &= \frac{p}{m} \bbE_\eps \sup_{f\in\cF} \sup_{g\in\cG} \sum_{i=1}^m \eps_i (f(z_i) - g(z_i))_{y'}
    \\
    &= \frac{p}{m}\bbE_\eps \sup_{f\in\cF} \sum_{i=1}^m \eps_i f(z_i)_{y'}
     + \frac{p}{m}\bbE_\eps \sup_{g\in\cG} \sum_{i=1}^m -\eps_i g(z_i)_{y'}
    \\
    &= p \Rad_m(\{ z\mapsto f(z)_{y'} : f\in \cF\})
    + p \Rad_m(\{ z\mapsto g(z)_{y'} : g\in \cG \}).
  \end{align*}
\end{proof}

To prove \Cref{fact:main:fix:3}, it still remains to collect a few convenient properties
of the softmax.

\begin{lemma}
  \label{fact:phig}
For any $v\in\R^k$ and $y\in\{1,\ldots,k\}$,
\[
    2(1- \phig(v))_y
\geq \1[ y \neq \argmax_i v_i ].
  \]
  Moreover, for any functions $\cF$ with $\cF\ni f : \cX \to \R^k$,
  \[
    \Rad_n\del{\cbr{ (x,y)\mapsto \phig(f(x))_y : f\in\cF}}
    =
    \tcO\del{ \frac {\sqrt k} \gamma \Rad_n(\cF) }.
  \]
\end{lemma}
\begin{proof}
For the first property, let $v\in\R^k$ be given, and consider two cases.
  If $y = \argmax_i v_i$, then $\phig(v) \in [0,1]^k$ implies
  \[
    2\del{1-\phig(v)}_y \geq 0 = \1[y \neq \argmax_i v_i].
  \]
  On the other hand, if $y \neq \argmax_i v_i$,
  then $\phig(v)_y \leq 1/2$,
  and
  \[
    2\del{1-\phig(v)}_y \geq 1 = \1[y \neq \argmax_i v_i].
  \]

  The second part follows from a multivariate Lipschitz composition lemma for Rademacher
  complexity due to \citet[Theorem 1]{dylan_multivariate_lipschitz}; all that remains to prove
  is that $v\mapsto \phig(v)_y$ is $(1/\gamma)$-Lipschitz with respect to the $\ell_\infty$ norm
  for any $v\in\R^k$ and $y\in[k]$.  To this
  end, note that
  \[
    \frac {\dif}{\dif v_y} \phig(v)_y = \frac {\exp(v/\gamma)_y \sum_{j\neq y} \exp(v/\gamma)_j}{\gamma (\sum_j \exp(v/\gamma)_j)^2}
    ,
    \qquad
    \frac {\dif}{\dif v_{i\neq y}} \phig(v)_y = - \frac {\exp(v/\gamma)_y \exp(v/\gamma)_i}{\gamma (\sum_j \exp(v/\gamma)_j)^2},
  \]
  and therefore
  \[
    \enVert{ \nabla \phig(v)_y }_1
    =
    \frac {2\exp(v/\gamma)_y \sum_{j\neq y} \exp(v/\gamma)_j}{\gamma (\sum_j \exp(v/\gamma)_j)^2}
    \leq
    \frac 1 \gamma,
  \]
  and thus, by the mean value theorem, for any $u\in\R^k$ and $v\in\R^k$,
  there exists $z\in [u,v]$ such that
  \[
    \envert{ \phig(v)_y - \phig(u)_y }
    =
    \envert{ \ip{ \nabla \phig(z)_y }{ v - u } }
    \leq
    \|v-u\|_\infty \cdot \enVert{ \nabla \phig(v)_y }_1
    \leq
    \frac 1 {\gamma}
    \|v-u\|_\infty,
  \]
  and in particular $v\mapsto \phig(v)/y$ is $(1/\gamma)$-Lipschitz with respect to
  the $\ell_\infty$ norm.
  Applying the aforementioned Lipschitz composition rule
  \citep[Theorem 1]{dylan_multivariate_lipschitz},
  \[
    \Rad_n\del{\cbr{ (x,y)\mapsto \phig(f(x))_y : f\in\cF}}
    =
    \tcO\del{ \frac {\sqrt k} \gamma \Rad_n(\cF) }.
  \]
\end{proof}

\Cref{fact:main:fix:3} now follows by combining \Cref{fact:main:fix,fact:phig}.

\begin{proof}[Proof of \Cref{fact:main:fix:3}]
  Define $\psi := 1 - \phig$.
  The bound follows by instantiating \Cref{fact:main:fix} with $p=1$ and
  the two function
  classes
  \[
    \cQ_{\cF} := \{ (x,y) \mapsto \psi(f(x)_y) : f\in\cF \}
    \qquad\textup{and}\qquad
    \cQ_{\cG} := \{ (x,y) \mapsto \psi(g(x)_y) : g\in\cG \},
  \]
  combining its simplified Rademacher upper bounds with the estimates for
  $\Rad_m(\cQ_{\cF})$ and $\Rad_m(\cQ_{\cG})$ and $\Rad_n(\cQ_{\cG})$ from \Cref{fact:phig},
  and by using \Cref{fact:phig} to lower bound the left hand side with
  \[
    \bbE \psi(f(x))_y =
    \bbE (1 -\phig(f(x))_y) \geq \frac 1 2 \1\sbr[2]{ \argmax_{y'}f(x)_{y'}\neq y},
  \]
  and lastly noting that
  \[
    \frac 1 m \sum_{i=1}^m \|\psi(f(z_i)) - \psi(g(z_i))\|_1
    =
    \frac 1 m \sum_{i=1}^m \|1 - \phig(f(z_i)) - 1 + \phig(g(z_i))\|_1
    =
    \Phi_{\gamma,m}(f,g).
  \]
\end{proof}

To complete the proofs for \Cref{sec:abstract}, it remains to handle the data augmentation
error, namely the term $\|\nicefrac{\dif\mu_{\cX}}{\dif\nu_n}\|_\infty$.  This proof
uses the following result about Gaussian kernel density estimation.

\begin{lemma}[name={See \citep[Theorem 2 and Remark 8]{jiang-kde}}]
  \label{fact:kde}
  Suppose density $p$ is $\alpha$-H\"older continuous, meaning
  $|p(x) - p(x')| \leq C_\alpha \|x-x'\|^\alpha$ for some $C_\alpha \geq 0$
  and $\alpha\in[0,1]$.
  There there exists a constant $C \geq 0$, depending on $\alpha$, $C_\alpha$,
  $\max_{x\in\R^d} p(x)$, and the dimension, but independent of the sample size,
  so that
  with probability at least $1-1/n$,
  the Gaussian kernel density estimate with bandwidth $\sigma^2 I$ where
  $\sigma = n^{-1/(2\alpha + d)}$ satisfies
  \[
    \sup_{x\in\R^d} |p(x) - p_n(x)| \leq C \sqrt{\frac {\ln(n)}{n^{2\alpha/(2\alpha+d)}}}.
  \]
\end{lemma}

The proof of \Cref{fact:augmentation} follows.

\begin{proof}[Proof of \Cref{fact:augmentation}]
  The proposed data augmentation measure $\nu_n$ has a density $\skde$ over $[0,1]^d$,
  and it has the form
  \[
    \skde(x) = \beta + (1-\beta)p_n(x),
  \]
  where $\beta = 1/2$, and $p_n$ is the kernel density estimator as described in
  \Cref{fact:kde},  whereby
  \[
    |p_n(x) - p(x)| \leq \eps_n := \cO \del{\frac {\sqrt{\ln n}}{n^{\alpha / (2\alpha + d)}}}.
  \]
  The proof proceeds to bound $\|\nicefrac{\dif\mu_{\cX}}{\dif\nu_n}\|_\infty
  = \|\nicefrac {p}{\skde}\|_\infty$ by considering three cases.
  \begin{itemize}
    \item If $x\not\in[0,1]^d$, then $p(x)=0$ by the assumption on the support of $\mu_\cX$,
      whereas $\skde(x) \geq p_n(x)/2 >0$, thus $p(x)/\skde(x) = 0$.

    \item
      If $x\in[0,1]^d$ and $p(x) \geq 2\eps_n$, then
      $\skde(x) \geq (1-\beta) p(x) - \eps_n) \geq \eps_n/2$, and
      \begin{align*}
        \frac {p(x)}{\skde(x)}
        &=
        1 + \frac {p(x) - \skde(x)}{\skde(x)}
        \\
        &\leq
        1 + \frac {\beta p(x)}{\skde(x)} + \frac{(1-\beta)|p(x) - p_n(x)|}{\skde(x)}
        \\
        &\leq
        1
        + \frac {\beta p(x)}{(1-\beta)(p(x) - \eps_n)}
        + \frac{(1-\beta)\eps_n}{\eps_n/2}
        \\
        &\leq
        1
        + \frac {\beta}{(1-\beta)(1 - \eps_n/p(x))}
        + 1
        \\
        &\leq
        4.
      \end{align*}

    \item
      If $x\in[0,1]^d$ and $p(x) < 2\eps_n$,
      since $\skde(x) \geq \beta = 1/2$,
      then
      \[
        \frac {p(x)}{\skde(x)} < \frac {2\eps_n}{\beta} = 4\eps_n.
      \]
  \end{itemize}
  Combining these cases, $\|\nicefrac{\dif\mu_{\cX}}{\dif\nu_n}\|_\infty
  = \|\nicefrac{p}{\skde}\|_\infty
  \leq
  \max\{4, 4\eps_n\}
  \leq
  4 + 4\eps_n$.
\end{proof}

\section{Replacing softmax with standard margin (ramp) loss}
\label{sec:margins}

The proof of \Cref{fact:main:fix:3} was mostly a reduction to
\Cref{fact:main:fix}, which mainly needs bounded functions; for the Rademacher complexity
estimates, the Lipschitz property of $\phig$ was used.
As such, the softmax can be replaced with
the $(1/\gamma)$-Lipschitz ramp loss as is standard from margin-based generalization
theory (e.g., in a multiclass version as appears in \citep{spec}).
Specifically, define $\cM_{\gamma} :\R^k \to [0,1]^k$ for any coordinate $j$ as
\[
  \cM_{\gamma}(v)_j := \ell_{\gamma}( v_j - \argmax_{y'\neq j} v_{y'}),
  \qquad\textup{where }
  \ell_{\gamma}(z) := \begin{cases}
    1 &z \leq 0,
    \\
    1 - \frac {z}{\gamma}
      &z\in  (0,\gamma),
      \\
    0 &z\geq \gamma.
  \end{cases}
\]
We now have $\1[\argmax_{y'} f(x)_{y'}] \leq \cM_{\gamma}(f(x))_y$ without a factor
of $2$ as in \Cref{fact:phig}, and can plug it into the general lemma in \Cref{fact:main:fix}
to obtain the following corollary.

\begin{corollary}
  Let temperature (margin!) parameter $\gamma > 0$ be given,
  along with sets of multiclass predictors $\cF$ and $\cG$.
  Then with probability at least $1-2\delta$ over an iid draw of data $((x_i,y_i))_{i=1}^n$ from
  $\mu$ and $(z_i)_{i=1}^n$ from $\nu_n$,
  every $f\in\cF$ and $g\in\cG$ satisfy
  \begin{align*}
    \Pr[\argmax_{y'} f(x)_{y'} \neq y]
    &\leq
      \enVert{ \frac {\dif \mu_{\cX}}{\dif \nu_n} }_{\infty}
      \frac 1 m \sum_{i=1}^m \|\cM_\gamma(f) - \cM_\gamma(g)\|_1
      +
      \frac 1 n \sum_{i=1}^n \cM_\gamma(g(x_i))_{y_i}
      \\
    &
      +
      \tcO\del[3]{
        \frac {k^{3/2}}{\gamma}
      \enVert{ \frac {\dif \mu_{\cX}}{\dif \nu_n} }_{\infty}
      \del{ \Rad_m(\cF) + \Rad_m(\cG)}
    + \frac {\sqrt k}{\gamma}\Rad_n(\cG) }
      \\
      &
      +
    3 \sqrt{\frac{\ln(1/\delta)}{2n}}
      \del{ 1 + 
          \enVert{ \frac {\dif \mu_{\cX}}{\dif \nu_n} }_{\infty}
          \sqrt{\frac n m }
      }.
    \end{align*}
\end{corollary}
\begin{proof}
  Overload function composition notation to sets of functions, meaning
  \[
    \cM_{\gamma}\circ \cF = \cbr{
      (x,y) \mapsto \cM_\gamma(f(x))_y : f\in\cF
    }.
  \]
  First note that $\cM_\gamma$ is $(2/\gamma)$-Lipschitz with respect to the $\ell_{\infty}$
  norm, and thus, applying the multivariate Lipschitz composition lemma for
  Rademacher complexity \citep[Theorem 1]{dylan_multivariate_lipschitz} just as in the proof
  for the softmax in \Cref{fact:phig},
  \[
    \Rad_m(\cM_{\gamma}\circ \cF) = \tcO\del{ \frac {2\sqrt{k}}{\gamma} \Rad_m(\cF) },
  \]
  with similar bounds for $\Rad_m(\cM_\gamma\circ \cG)$ and $\Rad_n(\cM_\gamma\circ \cG)$.
  The desired statement now follows by combining these Rademacher complexity bounds
  with \Cref{fact:main:fix:3} applied to $\cM_\gamma\circ\cF$ and $\cM_\gamma \circ \cG$,
  and additionally using $\1[\argmax_{y'} f(x)_{y'}\neq y] \leq \cM_{\gamma}(f(x))_y$.
\end{proof}

\section{Sampling tools}

The proofs of \Cref{fact:snb} and \Cref{fact:sparsify:2} both make heavy use of sampling.

\begin{lemma}[Maurey \citep{pisier1980remarques}]
  \label{fact:maurey}
  Suppose random variable $V$ is almost surely
  supported on a subset $S$ of some Hilbert space,
  and let $(V_1,\ldots,V_k)$ be $k$ iid copies of $V$.  Then
  there exist $(\hV_1,\ldots,\hV_k) \in S^k$ with
  \begin{align*}
    \enVert{ \bbE V - \frac 1 k \sum_i \hV_i }_\tF^2
    \leq
    \E_{V_1,\ldots,V_k}
    \enVert{ \bbE V - \frac 1 k \sum_i V_i }_\tF^2
    =
    \frac 1 {k}
    \sbr{
      \bbE \|V\|_\tF^2 - \|\bbE V\|_\tF^2
    }
    \leq
    \frac 1 k \bbE \|V\|^2_\tF
    \leq
    \frac 1 k \sup_{\hV\in S} \|\hV\|^2_\tF.
  \end{align*}
\end{lemma}
\begin{proof}[Proof of \Cref{fact:maurey}]
  The first inequality is via the probabilistic method.
  For the remaining inequalities, by expanding the square multiple times,
  \begin{align*}
    \E_{V_1,\ldots,V_k}
    \enVert{ \bbE V - \frac 1 k \sum_i V_i }_\tF^2
    &\leq
    \E_{V_1,\ldots,V_k}
    \frac 1 {k^2}\sbr{
      \sum_i \enVert{ \bbE V - V_i }_\tF^2
      +
      \sum_{i\neq j} \ip{\bbE V - V_i}{\bbE V - V_j}
    }
    \\
    &=
    \frac 1 {k}
    \bbE_{V_1}
    \enVert{
      V_1 - \bbE V
    }_\tF^2
=
    \frac 1 {k}
    \sbr{
      \bbE \|V\|_\tF^2 - \|\bbE V\|_\tF^2
    }
    \leq
    \frac 1 k \bbE \|V\|^2_\tF
    \leq
    \frac 1 k \sup_{\hV\in S} \|\hV\|^2_\tF.
  \end{align*}
\end{proof}

A first key application of \Cref{fact:maurey} is to sparsify products,
as used in \Cref{fact:sparsify:2}.

\begin{lemma}
  \label{fact:maurey:matrix_mult}
  Let matrices $A \in \R^{d\times m}$
  and $B \in \R^{n\times m}$ be given,
  along with sampling budget $k$.
  Then there exists a selection $(i_1,\ldots,i_k)$ of indices and a corresponding
  diagonal \emph{sampling matrix} $M$ with at most $k$ nonzero entries satisfying
  \[
    M := \frac {\|A\|_\tF^2} k \sum_{j=1}^k \frac{\be_{i_j} \be_{i_j}^\T}{\|A\be_{i_j}\|^{2}}
    \qquad\text{and}\qquad
    \enVert{
      AB^\T - A M B^\T
    }^2
    \leq
    \frac 1 {k}
    \|A\|^2\|B\|^2.
  \]
\end{lemma}
\begin{proof}[Proof of \Cref{fact:maurey:matrix_mult}]
  For convenience, define columns $a_i := A\be_i$ and $b_i:=B\be_i$ for $i\in\{1,\ldots,m\}$.
  Define \emph{importance weighting} $\beta_i := (\nicefrac {\|a_i\|}{\|A\|_\tF})^2$,
  whereby $\sum_i \beta_i = 1$,
  and let $V$ be a random variable with
  \[
    \Pr\sbr{
      V = \beta_i^{-1} a_i b_i^\T
    } = \beta_i,
  \]
  whereby
  \begin{align*}
    \bbE V
    &= \sum_{i=1}^m \beta_i^{-1} a_i b_i^\T \beta_i
    = \sum_{i=1}^m (A\bfe_i) (B\bfe_i)^\T
    = A \sbr[3]{ \sum_{i=1}^m \bfe_i \bfe_i^\T } B^\T
    = A \sbr{ I } B^\T
    = AB,
    \\
    \bbE \|V\|^2
    &= \sum_{i=1}^m \beta_i^{-2} \|a_i b_i^\T\|_\tF^2 \beta_i
    = \sum_{i=1}^m \beta_i^{-1} \|a_i\|^2 \| b_i\|^2
    = \sum_{i=1}^m \|A\|^2_\tF \| b_i \|^2
    = \|A\|_\tF^2 \cdot \|B\|_\tF^2.
  \end{align*}
  By \Cref{fact:maurey}, there exist indices $(i_1,\ldots,i_k)$
  and matrices $\hV_j := \beta_{i_j}^{-1} a_{i_j} b_{i_j}^\T$ with
  \[
    \enVert{
      AB^\T - \frac 1 k \sum_j \hV_j
    }^2
    \leq
    \enVert{
      \bbE V - \frac 1 k \sum_j \hV_j
    }^2
    =
    \frac 1 {k}
    \sbr{
      \|A\|_\tF^2 \|B\|_\tF^2 - \|AB\|_\tF^2
    }
    \leq
    \frac 1 {k}
    \|A\|_\tF^2
    \|B\|_\tF^2.
  \]
  To finish, by the definition of $M$,
  \[
    \frac 1 k \sum_j \hV_j
    = \frac 1 k \sum_j \beta_{i_j}^{-1} (A\bfe_{i_j}) (B\bfe_{i_j})^\T
    = A \sbr{ \frac 1 k \sum_j \beta_{i_j}^{-1} \bfe_{i_j} \bfe_{i_j} ^\T } B^\T
    = A \sbr{ M  } B^\T.
  \]
\end{proof}

A second is to cover the set of matrices $W$ satisfying a norm bound $\|W^\T\|_{2,1}\leq r$.
The proof here is more succinct and explicit than the one in \citep[Lemma 3.2]{spec}.

\begin{lemma}[name={See also \citep[Lemma 3.2]{spec}}]
  \label{fact:21}
  Let norm bound $r\geq 0$, $X\in\R^{n\times d}$, and integer $k$ be given.
  Define a family of matrices
  \[
    \cM := \cbr{ \frac {r\|X\|_\tF} k
      \sum_{l=1}^k \frac {s_l\bfe_{i_l}\bfe_{j_l}^\T}{\|X\bfe_{j_l}\|}
    : s_l\in\{\pm 1\}, i_l \in \{1,\ldots, n\}, j_l \in  \{1,\ldots,d\}}.
  \]
  Then
  \[
    |\cM| \leq (2nd)^k,
    \qquad
    \sup_{\|W^\T\|_{2,1}\leq r} \min_{\hW\in\cM}
    \|WX^\T - \hW{}X^\T \|_\tF^2 \leq \frac {r^2\|X\|_\tF^2}{k}.
  \]
\end{lemma}
\begin{proof}
  Let $W\in\R^{m\times d}$ be given with $\|W^\T\|_{2,1} \leq r$.
  Define $s_{ij} := W_{ij}/|W_{ij}|$, and note
  \begin{align*}
    WX^\T
    = \sum_{i,j} \bfe_i\bfe_i^\T W \bfe_j\bfe_j^\T X^\T
    = \sum_{i,j} \bfe_i W_{ij} (X\bfe_j)^\T
    = \sum_{i,j} \underbrace{\frac{|W_{ij}| \|X\bfe_j\|_2}{r\|X\|_\tF}}_{=:q_{ij}}
    \underbrace{\frac{r \|X\|_\tF s_{ij}\bfe_i (X\bfe_j)^\T }{\|X\bfe_j\|}}_{=:U_{ij}}.
  \end{align*}
  Note by Cauchy-Schwarz that
  \[
    \sum_{i,j} q_{ij}
    \leq \frac {1}{r\|X\|_\tF} \sum_i \sqrt{\sum_j W_{ij}^2}\|X\|_\tF
    = \frac {\|W^\T\|_{2,1} \|X\|_\tF}{r\|X\|_\tF} \leq 1,
  \]
  potentially with strict inequality, thus $q$ is not a probability vector.
  To remedy this, construct probability vector $p$ from $q$ by adding in, with equal weight,
  some $U_{ij}$ and its negation, so that the above summation form of $WX^\T$ goes through
  equally with $p$ and with $q$.

  Now define iid random variables $(V_1,\ldots,V_k)$, where
  \begin{align*}
    \Pr[V_l = U_{ij}] &= p_{ij},
    \\
    \bbE V_l &= \sum_{i,j} p_{ij} U_{ij} = \sum_{i,j} q_{ij} U_{ij} = WX^\T,
    \\
    \|U_{ij}\| 
             &= \enVert{\frac {s_{ij} \bfe_i (X\bfe_j)}{\|X\bfe_j\|_2} }_\tF \cdot r \|X\|_\tF
             = |s_{ij}|\cdot\|\bfe_i\|_2 \cdot\enVert{\frac{X\bfe_j}{\|X\bfe_j\|_2} }_2
             \cdot r \|X\|_\tF
             =
             r \|X\|_\tF,
             \\
    \bbE \|V_l\|^2
             &= \sum_{i,j} p_{ij} \|U_{ij}\|^2
             \leq \sum_{ij} p_{ij} r^2 \|X\|_\tF^2
             = r^2 \|X\|_\tF^2.
  \end{align*}
  By \Cref{fact:maurey}, there exist $(\hV_1,\ldots,\hV_k)\in S^k$ with
  \[
    \enVert{ WX^\T - \frac 1 k \sum_l \hV_l }^2
    \leq
    \bbE \enVert{ \bbE V_1  - \frac 1 k \sum_l V_l }^2
    \leq
    \frac 1 k
    \bbE \|V_1\|^2
    \leq
    \frac {r^2 \|X\|_\tF^2} k.
  \]
  Furthermore, the matrices $\hV_l$ have the form
  \[
    \frac 1 k \sum_l \hV_l
    = \frac 1 k \sum_l \frac {s_l\be_{i_l} (X\bfe_{j_l})^\T}{\|X\bfe_{j_l}\|}
    = \sbr{ \frac 1 k \sum_l \frac {s_l\be_{i_l} \bfe_{j_l}^\T}{\|X\bfe_{j_l}\|} } X^\T
    =: \hW X^\T,
  \]
  where $\hW\in\cM$.
  Lastly, note $|\cM|$ has cardinality at most $(2nd)^k$.
\end{proof}

\section{Proofs for \Cref{sec:concrete}}

The bulk of this proof is devoted to establishing the Rademacher bound for computation
graphs in \Cref{fact:snb}; thereafter, as mentioned in \Cref{sec:analysis},
it suffices to plug this bound and the data augmentation bound in \Cref{fact:augmentation}
into \Cref{fact:main:fix:3}, and apply a pile of union bounds.

As mentioned in \Cref{sec:analysis}, this 
proof follows the scheme laid out in \citep{spec},
with simplifications due to the removal of ``reference matrices'' and some norm generality.

\begin{proof}[Proof of \Cref{fact:snb}]
  Let cover scale $\eps$ and per-layer scales $(\eps_1,\ldots,\eps_L)$ be given;
  the proof will develop a covering number parameterized by these per-layer scales,
  and then optimize them to derive the final covering number in terms of $\eps$.
  From there, a Dudley integral will give the Rademacher bound.

  Define $\tb_i := b_i\sqrt{n}$ for convenience.
  As in the statement, recursively define
  \[
    X_0^\T := X^\T,
    \qquad
    X_i^\T := \sigma_i\del{ [ W_i \Pi_i D_i \mc F_i ] X_{i-1}^\T }.
  \]
  The proof will recursively construct an analogous cover via
  \[
    \hX_0^\T := X^\T,
    \qquad
    \hX_i^\T := \sigma_i\del{ [ \hW_i \Pi_i D_i \mc F_i ] \hX_{i-1}^\T },
  \]
  where the choice of $\hW_i$ depends on $\hX_{i-1}$, and thus the total cover
  cardinality will product (and not simply sum) across layers.  Specifically,
  the cover $\cN_i$ for $\hW_i$ is given by \Cref{fact:21} by plugging in
  $\|\Pi_i D_i \hX_{i-1}^\T\|_\tF \leq \tb_i$,
  and thus it suffices to choose
  \[
    \textup{cover cardinality } k := \frac {r_i^2 \tb_i^2}{\eps_i^2},
    \qquad
    \textup{whereby }
    \min_{\hW_i \in\cN_i} \|W_i \Pi_i D_i \hX_{i-1}^\T - \hW_i \Pi_i D_i \hX_{i-1}^\T\|
    \leq \eps_i.
  \]
  By this choice (and the cardinality estimate in \Cref{fact:21},
  the full cover $\cN$ satisfies
  \[
    \ln|\cN|
    = \sum_i \ln |\cN_i|
    \leq
    \sum_i \frac {r_i^2 \tb_i^2}{\eps_i^2} \ln(2m^2).
  \]

  To optimize the parameters $(\eps_1,\ldots,\eps_L)$, the first step is to show via
  induction that
  \[
    \|X_i^\T - \hX_i^\T\|_\tF \leq \sum_{j \leq i} \eps_j \rho_j \prod_{l=j+1}^i s_l \rho_l.
  \]
  The base case is simply $\|X_0^\T - \hX^\T\| = \|X^\T - X^\T\| = 0$, thus consider
  layer $i>0$.
  Using the inductive formula for $\hX_i$ and the cover guarantee on $\hW_i$,
  \begin{align*}
    \enVert{ X_i^\T - \hX_i^\T }
    &=
    \enVert{ \sigma_i([ W_i \Pi_i D_i \mc F_i ] X_{i-1}^\T)
    - \sigma_i([\hW_i \Pi_i D_i \mc F_i] \hX_{i-1}^\T) }
    \\
    &\leq
    \rho_i\enVert{ [ W_i \Pi_i D_i \mc F_i ] hX_{i-1}^\T
    - [\hW_i \Pi_i D_i \mc F_i] \hX_{i-1}^\T }
    \\
    &\leq
    \rho_i\enVert{ [ W_i \Pi_i D_i \mc F_i ] X_{i-1}^\T
    - [W_i \Pi_i D_i \mc F_i] \hX_{i-1}^\T }
    +
    \rho_i\enVert{ [ W_i \Pi_i D_i \mc F_i ] \hX_{i-1}^\T
    - [\hW_i \Pi_i D_i \mc F_i] \hX_{i-1}^\T }
    \\
    &\leq
    \rho_i\enVert{ [ W_i \Pi_i D_i \mc F_i ] }_2
    \enVert{ X_{i-1}^\T - \hX_{i-1}^\T }
    +
    \rho_i\enVert{ [ (W_i - \hW_i) \Pi_i D_i \hX_{i-1}^\T \mc (F_i - F_i)\hX_{i-1}^\T ] }
    \\
    &\leq
    s_i
    \rho_i\sum_{j \leq i-1} \eps_j \rho_j \prod_{l=j+1}^{i-1} s_l \rho_l
    +
    \rho_i\enVert{(W_i - \hW_i) \Pi_i D_i \hX_{i-1}^\T}
    \\
    &\leq
    \sum_{j \leq i-1} \eps_j\rho_j \prod_{l=j+1}^{i} s_l \rho_l
    +
    \rho_i \eps_i
    \leq
    \sum_{j \leq i} \eps_j \rho_j \prod_{l=j+1}^i s_l \rho_l.
  \end{align*}
  To balance $(\eps_1,\ldots,\eps_L)$, it suffices to minimize a Lagrangian corresponding
  to the cover size subject to an error constraint, meaning
  \[
    L(\vec\eps, \lambda) = \sum_{i=1}^L \frac {\alpha_i}{\eps_i^2}
    + \lambda \del{ \sum_{i=1}^L \eps_i \beta_i - \eps}
    \qquad\textup{where }
    \alpha_i := r_i^2 \tb_i^2 \ln(2m^2),
    \quad
    \beta_i := \rho_i \prod_{l = i+1}^L s_l \rho_l,
  \]
  whose unique critical point for $\vec\eps>0$ implies the choice
  \[
    \eps_i := \frac 1 Z \del{ \frac{2\alpha_i}{\beta_i} }^{1/3}
    \qquad\textup{where }
    Z := \frac 1 {\eps} \sum_i (2\alpha_i\beta_i^2)^{1/3},
  \]
  whereby $\|X_L^\T - \hX_L^\T \|\leq \eps$ automatically, and
  \begin{align*}
    \ln|\cN|
    &\leq Z^2 \sum_i \frac {r_i^2 \tb_i^2\ln(2m^2)}{(2\alpha_i/\beta_i)^{2/3}}
    \\
    &=
    \frac 1 {\eps^2 2^{2/3}} \sbr{ 2 \sum_i r_i^{2/3} \tb_i^{2/3} \beta_i^{2/3} \ln(2m^2)^{1/3} }^2
    \sum_i r_i^{2/3} \tb_i^{2/3}  \ln(2m^2)^{1/3} \beta_i^{2/3}
    \\
    &=
    \frac {2^{4/3} \ln(2m^2)} {\eps^2}
    \sbr[4]{ \sum_i \del[3]{r_i \tb_i \rho_i \prod_{l=i+1}^L s_l\rho_l}^{2/3} }^3
    =:
    \frac {\tau^2}{\eps^2},
  \end{align*}
  as desired, with $\tau$ introduced for convenience in what is to come.

  For the Rademacher complexity estimate, by a standard Dudley entropy integral
  \citep{shai_shai_book}, setting $\hat\tau := \max\{\tau, 1/3\}$ for convenience,
  \[
    n\Rad(\cG)
    \leq \inf_{\zeta} 4 \zeta \sqrt{n} + 12 \int_{\zeta}^{\sqrt n} \frac {\sqrt{\hat\tau}{\eps}}\dif\eps
    = \inf_{\zeta} 4 \zeta \sqrt{n} + 12 \hat\tau \left. \ln(\eps)\right|_\zeta^{\sqrt n}
      = \inf_{\zeta} 4 \zeta \sqrt{n} + 12 \hat\tau (\ln \sqrt{n} - \ln \zeta),
  \]
  which is minimized at $\zeta = 3\hat\tau /\sqrt{n}$, whereby
  \[
    n\Rad(\cG)
    \leq 12 \hat\tau + 6\hat\tau\ln n - 12\hat\tau \ln(3\hat\tau / \sqrt{n})
    = 12\hat\tau (1 - \ln(3\hat\tau))
    \leq 12 \hat \tau
    \leq 12 \tau + 4.
  \]
\end{proof}

This now gives the proof of \Cref{fact:pollard:snb}.

\begin{proof}[Proof of \Cref{fact:pollard:snb}]
  With \Cref{fact:main:fix:3}, \Cref{fact:augmentation}, and \Cref{fact:snb}
  out of the way, the main work of this
  proof is to have an infimum over distillation network hyperparameters
  $(\vec b, \vec r, \vec s)$ on the right hand side,
  which is accomplished by dividing these
  hyperparameters into countably many shells, and unioning over them.

  In more detail, divide $(\vec b, \vec r, \vec s)$ into shells as follows.
  Divide each $b_i$ and $r_i$ into shells of radius increasing by one, meaning
  meaning for example the first shell for $b_i$ has $b_i \leq 1$, and the $j$th shell
  has $b_i \in (j-1, j]$, and similarly for $r_i$; moreover, associate the $j$th shell
  with prior weight $q_j(b_i) := (j(j+1))^{-1}$, whereby $\sum_{j\geq 1} q_j(b_i) = 1$.
  Meanwhile, for $s_i$ use a finer grid where the first shell has $s_i \leq 1/L$, and the $j$th
  shell has $s_i \in ((j-1)/L, j/L)$, and again the prior weight is $q_j(s_i) = (j(j+1))^{-1}$.
  Lastly, given a full set of grid parameters $(\vec b, \vec r, \vec s)$, associate prior
  weight $q(\vec b, \vec r, \vec s)$ equal to the product of the individual prior weight, whereby
  the sum of the prior weights over the entire product grid is $1$.  Enumerate this grid in
  any way, and define failure probability
  $\delta(\vec b, \vec r, \vec s) := \delta \cdot q(\vec b, \vec r, \vec s)$.

  Next consider some fixed grid shell with parameters $(\vec b', \vec r', \vec s')$ and
  let $\cH$ denote the set of networks for which these parameters form the tightest shell,
  meaning that for any $g\in\cH$ with  parameters $(\vec b, \vec r, \vec s)$,
  then $(\vec b', \vec r', \vec s') \leq (\vec b + 1, \vec r + 1, \vec s + 1)$ component-wise.
  As such, by \Cref{fact:main:fix:3}, with probability at least
  $1-\delta(\vec b', \vec r', \vec s')$, each $g\in\cH$ satisfies
  \begin{align*}
    \Pr[\argmax_{y'} f(x)_{y'} \neq y]
    &\leq 2 
      \enVert{ \frac {\dif \mu_{\cX}}{\dif \nu_n} }_{\infty}
      \Phi_{\gamma,m}(f, g)
      +
      \frac 2 n \sum_{i=1}^n (1 - \phig(g(x_i))_{y_i}
      \\
    &
      +
      \tcO\del[3]{
        \frac {k^{3/2}}{\gamma}
      \enVert{ \frac {\dif \mu_{\cX}}{\dif \nu_n} }_{\infty}
      \del{ \Rad_m(\cF) + \Rad_m(\cH)}
    + \frac {\sqrt k}{\gamma}\Rad_n(\cH) }
      \\
      &
      +
      6 \sqrt{\frac{\ln(q(\vec b', \vec r', \vec s')) + \ln(1/\delta)}{2n}}
      \del{ 1 + 
          \enVert{ \frac {\dif \mu_{\cX}}{\dif \nu_n} }_{\infty}
          \sqrt{\frac n m }
      }.
  \end{align*}
  To simplify this expression, first note by \Cref{fact:snb} and the construction of the shells
  (relying in particular on the finer grid for $s_i$ to avoid a multiplicative factor $L$)
  that
  \begin{align*}
    \Rad_m(\cH)
    &= \tcO\sbr[4]{\frac {1}{\sqrt{n}}\del[4]{\sum_i \sbr[3]{r_i' b_i' \rho_i \prod_{l=i+1}^L s_l'\rho_l}^{2/3} }^{3/2}}
    \\
    &= \tcO\sbr[4]{\frac {1}{\sqrt{n}}\del[4]{ \sum_i \sbr[3]{(r_i+1) (b_i+1) \rho_i \prod_{l=i+1}^L (s_l+1/L)\rho_l}^{2/3} }^{3/2}}
    \\
    &= \tcO\sbr[4]{\frac {1}{\sqrt{n}}\del[4]{ \sum_i \sbr[3]{r_i b_i\rho_i \prod_{l=i+1}^L s_l\rho_l}^{2/3} }^{3/2}},
  \end{align*}
  and similarly for $\Rad_m(\cH)$ (the only difference being $\sqrt{m}$ replaces $\sqrt{n}$).
  Secondly, to absorb the term $\ln(q(\vec b', \vec r', \vec s'))$,
  noting that $\ln(a) \leq \ln(\gamma^2) + (a - \gamma^2) / (\gamma^2)$,
  and also using $\rho_i\geq 1$,
  then
  \begin{align*}
    \ln(q(\vec r', \vec b', \vec s'))
    &= \cO\del{ \ln\prod_i (r_i+1)^2 (b_i+1)^2 ((s_i +1) L)^2 }
    = \cO\del{ L \ln L + \ln\prod_i r_i^{2/3} b_i^{2/3} s_i^{2/3} }
    \\
    &
    = \tcO\del{L + \sum_i \ln( r_i^{2/3} b_i^{2/3}) + \ln \prod_i s_i^{2/3} }
    \\
    &
    = \tcO\del{L +  \ln(\gamma^2) + \frac {1}{\gamma^2} \sum_i \sbr{ r_i^{2/3} b_i^{2/3} + \prod_{l>i} s_l^{2/3} } }
    \\
    &
    = \tcO\del[4]{L + \frac {1}{\gamma^2}\sum_i \sbr[3]{r_i b_i \prod_{l=i+1}^L s_l}^{2/3} }
    \\
    &
    = \tcO\del[4]{L + \frac {1}{\gamma^2}\sum_i \sbr[3]{r_i b_i\rho_i \prod_{l=i+1}^L s_l\rho_l}^{2/3} }.
  \end{align*}
  Together,
  \begin{align*}
    \Pr[\argmax_{y'} f(x)_{y'} \neq y]
    &\leq 2 
      \enVert{ \frac {\dif \mu_{\cX}}{\dif \nu_n} }_{\infty}
      \Phi_{\gamma,m}(f, g)
      +
      \frac 2 n \sum_{i=1}^n (1 - \phig(g(x_i))_{y_i}
      \\
    &
      +
      \tcO\del[3]{
        \frac {k^{3/2}}{\gamma}
        \enVert{ \frac {\dif \mu_{\cX}}{\dif \nu_n} }_{\infty}
      \Rad_m(\cF) }
      +
      6 \sqrt{\frac{\ln(1/\delta)}{2n}}
      \del{ 1 + 
          \enVert{ \frac {\dif \mu_{\cX}}{\dif \nu_n} }_{\infty}
          \sqrt{\frac n m }
      }
      \\
    &+
      \tcO\del[3]{
        \frac {\sqrt{k}}{\gamma \sqrt n}
        \del{1 + k\enVert{ \frac {\dif \mu_{\cX}}{\dif \nu_n} }_{\infty} \sqrt{\frac{n}{m}}}
      \del[4]{ \sum_i \sbr[3]{r_i b_i\rho_i \prod_{l=i+1}^L s_l\rho_l}^{2/3} }^{3/2}
      }.
  \end{align*}
  Since $h\in\cH$ was arbitrary, the bound may be wrapped in $\inf_{g\in\cH}$.
  Similarly, unioning bounding away the failure probability for all shells,
  since this particular shell was arbitrary, an infimum over shells can be added, which
  gives the final infimum over $(\vec b, \vec r, \vec s)$.  The last touch is to
  apply \Cref{fact:augmentation} to bound
  $\|\nicefrac{\dif\mu_{\cX}}{\dif\nu_n}\|_\infty$.
\end{proof}

\section{Proof of stable rank bound, \Cref{fact:rad:sr}}

The first step is to establish the sparsification lemma in \Cref{fact:sparsify:2},
which in turn sparsifies each matrix product, cannot simply invoke
\Cref{fact:maurey:matrix_mult}: pre-processing is necessary to control the element-wise
magnitudes of the resulting matrix.  Throughout this section, define the stable rank
of a matrix $W$ as $\sr(W) := \|W\|_\tF^2 / \|W\|_2^2$ (or $0$ when $W = 0$).

\begin{lemma}
  \label{fact:maurey:matrix_mult:bounded}
  Let matrices $A \in \R^{d\times m}$
  and $B \in \R^{n\times m}$ be given,
  along with sampling budget $k$.
  Then there exists a selection $(i_1,\ldots,i_k)$ of indices and a corresponding
  diagonal \emph{sampling matrix} $M$ with at most $k$ nonzero entries satisfying
  \[
    M := \sum_{j=1}^k \frac{Z_{i_j}\be_{i_j} \be_{i_j}^\T}{\|a_{i_j}\|}
    \quad\text{where}\quad
    Z_{i_j} \leq \|A\|_\tF\sqrt{\frac{m}{k}},
    \qquad\text{and}\qquad
    \enVert{
      AB^\T - A M B^\T
    }^2
    \leq
    \frac 4 {k}
    \|A\|^2\|B\|^2.
  \]
\end{lemma}
\begin{proof}
  Let $\tau > 0$ be a parameter to be optimized later,
  and define a subset of indices $S := \{ i \in \{1,\ldots,m\} : \|A\be_i\| \geq \tau \}$,
  with $S^c := \{1,\ldots,m\} \setminus S$.  Let $A_\tau$ denote the matrix obtained
  by zeroing out columns not in $S$, meaning
  \[
    A_\tau := \sum_{i\in S} (A\be_i)\bfe_i^\T,
  \]
  whereby
  \[
    \| A B^\T - A_\tau B^\T \|_\tF
    \leq
    \| A-A_\tau\|\cdot \|B\|
    \leq \|B\| \sqrt{ \sum_{i\in S^c} \|A\be_i\|^2 } 
    \leq \tau \sqrt{m} \|B\|.
  \]
  Applying \Cref{fact:maurey:matrix_mult} to $A_\tau B^\T$ gives
  \[
    M := \frac {\|A_\tau\|^2}{k} \sum_{j=1}^k \frac {\be_{i_j}\be_{i_j}^\T}{\|A_\tau \be_{i_j}\|^2}
    = \sum_{j=1}^k \frac {Z_{i_j} \be_{i_j}\be_{i_j}^\T}{\|A_\tau \be_{i_j}\|}
    \qquad
    \textup{such that}
    \qquad
    \|A_\tau B^\T - A_\tau M B^\T\|^2
    \leq
    \frac {1}{k}\|A_\tau\|^2\|B\|^2,
  \]
  where $Z_{i_j}$ is specified by these equalities.
  To simplify, note $\|A_\tau\|\leq \|A\|$, and $A_\tau M = AM$.
  Combining the two inequalities,
  \[
    \|AB^\T - AMB^\T\|
    \leq
    \|AB^\T - A_\tau B^\T\| + \|A_\tau B^\T - A_\tau M B^\T\|
    \leq
    \tau \sqrt{m} \|B\|
    +
    \frac 1 {\sqrt k} \|A\| \|B\|.
  \]
  To finish, setting $\tau := \|A\| / \sqrt{mk}$ gives the bound,
  and ensures that the scaling term $Z_{i_j}$ satisfies, for any $i_j\in S$,
  \[
    Z_{i_j}
    =
    \frac {\|A_\tau\|^2}{k \|A_\tau\be_{i_j}\|}
    \leq
    \frac {\|A\|_\tF^2}{k\tau}
    =
    \|A\|_\tF\sqrt{\frac{m}{k}}.
  \]
\end{proof}

With this tool in hand, the proof of \Cref{fact:sparsify:2} is as follows.

\begin{proof}[Proof of \Cref{fact:sparsify:2}]
  Let $X_j$ denote the network output after layer $j$, meaning
  \[
    X_0^\T := X^\T,
    \qquad
    X_{j}^\T := \sigma_j(W_jX_{j-1}^\T),
  \]
  whereby
  \[
    \|X_j^\T\|_\tF
    = \|\sigma_j(W_jX_{j-1}^\T) - \sigma_j(0)\|_\tF
    \leq \|W_jX_{j-1}^\T\|_\tF
    \leq \|W_j\|_2 \|X_{j-1}^\T\|_\tF
    \leq \|X\|_\tF\prod_{i \leq j}\|W_i\|_2.
  \]
  The proof will inductively choose sampling matrices $(M_1,\ldots,M_L)$ as in
  the statement and construct
  \[
    \hX_0^\T := X^\T,
    \qquad
    \hX_j^\T := \Pi_j \sigma_j(W_j M_j \hX_{j-1}^\T),
  \]
  where $\Pi_j$ denotes projection onto the Frobenius-norm ball of radius
  $\|X\|_\tF \prod_{i \leq j}\|W_i\|_2$ (whereby $\Pi_j X_j = X_j$),
  satisfying
  \[
    \enVert{X_j - \hX_j}_\tF
    \leq
    \|X\|_\tF
    \sbr{ \prod_{p=1}^j \|W_p\|_2 }
    \sum_{i=1}^j
    \sqrt{\frac{\sr(W_i)}{k_i}},
  \]
  which gives the desired bound after plugging in $j=L$.

  Proceeding with the inductive construction, the base case
  is direct since $\hX_0 = X = X_0$ and $\enVert{ X_0 - \hX_0}_\tF = 0$,
  thus consider some $j>0$.  Applying \Cref{fact:maurey:matrix_mult:bounded}
  to the matrix multiplication $W_j \hX_{j-1}$ with $k_j$ samples, there exists a
  multiset of $S_j$ coordinates and a corresponding sampling matrix $M_j$,
  as specified in the statement,
  satisfying
  \begin{align*}
    \enVert{ W_j \hX_{j-1}^\T - W_j M_j \hX_{j-1}^\T }_\tF
    &\leq
    \frac 1 {\sqrt{k_j}} \|W_j\|_\tF \|\hX_{j-1}\|_\tF
\leq
    \frac 1 {\sqrt{k_j}} \|W_j\|_\tF
    \|X\|_\tF \prod_{i< j}\|W_i\|_2.
  \end{align*}
  Using the choice $\hX_j^\T := \Pi_j \sigma_j(W_j M_j \hX_{j-1}^\T)$,
  \begin{align*}
    \enVert{X_j - \hX_j}_\tF
    &=
    \enVert{\sigma_j(W_jX_{j-1}^\T) - \Pi_j \sigma_j(W_j M_j \hX_{j-1}^\T)}_\tF
    \\
    &\leq
    \enVert{W_jX_{j-1}^\T - W_j M_j \hX_{j-1}^\T}_\tF
    \\
    &=
    \enVert{W_jX_{j-1}^\T - W_j \hX_{j-1}^\T + W_j\hX_{j-1}^\T - W_j M_j \hX_{j-1}^\T}_\tF
    \\
    &\leq
    \enVert{W_jX_{j-1}^\T - W_j \hX_{j-1}^\T}_\tF
    + \enVert{W_j\hX_{j-1}^\T - W_j M_j \hX_{j-1}^\T}_\tF
    \\
    &\leq
    \enVert{W_j}_2\enVert{X_{j-1} - \hX_{j-1}}_\tF
    + \frac 1 {\sqrt{k_j}} \|W_j\|_\tF\|X\|_\tF \prod_{i<j}\|W_i\|_2
    \\
    &\leq
    \enVert{W_j}_2 \del{ \|X\|_\tF \sbr{ \prod_{i<j} \|W_i\|_2} \sum_{i<j} \sqrt{\frac{\sr(W_i)}{k_i}} }
    + \sqrt{\frac {\sr(W_j)} {k_j}} \|X\|_\tF \prod_{i\leq j}\|W_i\|_2
    \\
    &\leq
    \|X\|_\tF \sbr{ \prod_{i\leq j} \|W_i\|_2} \sum_{i\leq j} \sqrt{\frac{\sr(W_i)}{k_i}}
  \end{align*}
  as desired.
\end{proof}

To prove \Cref{fact:rad:sr} via \Cref{fact:sparsify:2},
the first step is a quick tool to cover matrices element-wise.

\begin{lemma}
  \label{fact:infty_cover}
  Let $\cA$ denote matrices with at most $k_2$ nonzero rows and $k_1$ nonzero columns,
  entries bounded in absolute value by $b$, and total number of rows and columns each at most $m$.
  Then there exists a cover set $\cM\subseteq \cA$ satisfying
  \[
    |\cM| \leq m^{k_1+k_2} \del{ \frac {2b \sqrt{k_1 k_2}}{\eps} }^{k_1k_2},
    \qquad\textup{and}\qquad
    \sup_{A\in\cA} \min_{\hA\in\cM} \|A - \hA\|_\tF \leq \eps.
  \]
\end{lemma}
\begin{proof}
  Consider some fixed set of $k_2$ nonzero rows and $k_1$ nonzero columns,
  and let $\cM_0$ denote the covering set obtained by gridding the $k_1\cdot k_2$
  entries at scale $\frac {\eps}{\sqrt{k_1 k_2}}$, whereby
  \[
    |\cM_0| \leq \del{ \frac {2b \sqrt{k_1 k_2}}{\eps} }^{k_1k_2}.
  \]
  For any $A\in\cA$ with these specific nonzero rows and columns,
  the $\hA\in\cM_0$ obtained by rounding each nonzero entry of $A$ to the nearest grid element
  gives
  \[
    \| A - \hA\|^2
    = \sum_{i,j} (A_{ij} - \hA_{ij})^2
    \leq \sum_{i,j} \del{\frac{\eps}{\sqrt{k_1k_2}}}^2
    = \eps^2  \sum_{i,j} \frac{1}{k_1k_2}
    = \eps^2.
  \]
  The final cover $\cM$ is now obtained by unioning copies of $\cM_0$ for all $\binom{m}{k_1}\binom{m}{k_2} \leq m^{k_1+k_2}$ possible submatrices of size $k_2\times k_1$.
\end{proof}

The proof of \Cref{fact:rad:sr} now carefully combines the preceding pieces.

\begin{proof}[Proof of \Cref{fact:rad:sr}]
  The proof proceeds in three steps, as follows.
\begin{enumerate}
    \item
      A covering number is estimate for sparsified networks, as output
      by \Cref{fact:sparsify:2}.
    \item
      A covering number for general networks is computed by balancing
      the error terms from \Cref{fact:sparsify:2} and its cover computed here.
    \item
      This covering number is plugged into a Dudley integral to obtain the desired
      Rademacher bound.
  \end{enumerate}

  Proceeding with this plan, let $(\hX_0^\T,\ldots,\hX_L^\T)$ be the layer outputs
  (and network input) exactly as provided by \Cref{fact:sparsify:2}.
  Additionally, define diagonal matrices $D_j := \sum_{l \stackrel{!}{\in}S_{j+1}}\be_l\be_l^\T$
  (with $D_L = I$,
  where the ``!'' denotes unique inclusion; these matrices capture the effect of the subsequent
  sparsification, and can be safely inserted after each $W_j$ without affecting
  $\hX_j^\T$, meaning
  \[
    \hX_j^\T
    = \Pi_j \sigma_j(W_j M_j \hX_{j-1}^\T) 
    = \Pi_j \sigma_j(D_j W_j M_j \hX_{j-1}^\T).
  \]

  Let per-layer cover precisions $(\eps_1,\ldots,\eps_L)$ be given, which will be
  optimized away later.
  This proof will inductively construct
  \[
    \tX_0^\T := X^\T,
    \qquad
    \tX_j^\T := \Pi_j \sigma_j(\tW_j \tX_{j-1}^\T),
  \]
  where $\tW_j$ is a cover element for $D_j W_j M_j$,
  and inductively satisfying
  \[
    \|\hX_j^\T - \tX_j^\T\|
    \leq
    \|X\|_\tF m^{j/2} \sum_{i\leq j} \eps_i \prod_{\substack{l \leq j\\l \neq i}}\|W_j\|_\tF.
  \]
  To construct the per-layer cover elements $\tW_j$,
  first note by the form of $M_j$ (and the scaling $Z_i$ provided by \Cref{fact:sparsify:2})
  that
  \[
    b:= \max_{i,l} (D_j W_jM_j)_{l,i}
    \leq \max_{i} \| W_jM_j\be_i\|
    \leq Z_i \frac {\|W_j \be_i\|}{\|W_j \be_i\|}
    \leq
    \|W_j\|_\tF \sqrt{\frac{m}{k_{j-1}}}.
  \]
  Consequently, by \Cref{fact:infty_cover},
  there exists a cover $\cC_j$ of matrices of the form $D_jW_jM_j$
  satisfying
  \[
    |\cC_j| \leq m^{k_j+k_{j-1}} \del{ \frac {2b \sqrt{k_j k_{j-1}}}{\eps_j} }^{k_jk_{j-1}}
    \leq m^{k_j+k_{j-1}} \del{ \frac {2\|W_j\|_\tF \sqrt{k_j m }}{\eps_j} }^{k_jk_{j-1}},
  \]
  and the closest cover element $\tW_j\cC_j$ to $D_jW_jM_j$ satisfies
  $\|D_j W_j M_j - \tW_j \|_\tF \leq \eps_j$.

  Proceeding with the induction,
  the base case has $\|\hX_0^\T - \tX_0^\T\| = \|X^\T - X^\T\|=0$, thus consider $j>0$.
  The first step is to estimate the spectral norm of $D_jW_jM_j$,
  which can be coarsely upper bounded via
  \[
    \|D_jW_jM_j\|_2^2
    \leq
    \|D_jW_jM_j\|_\tF^2
    \leq
    \sum_i \|W_j M_j \be_i\|^2
    \leq
    \sum_i \|W_j\|_\tF^2 \frac{m}{k_{j-1}}
    \leq
    \|W_j\|_\tF^2 m.
  \]
  By the form of $\hX_j$ and $\tX_j$,
  \begin{align*}
    \|\hX_j^\T - \tX_j^\T\|
    &=
    \|\Pi_j \sigma_j (D_jW_j M_j \hX_{j-1}^\T) - \Pi_j \sigma_j (\tW_j \tX_{j-1}^\T)\|
    \\
    &\leq
    \|D_jW_j M_j \hX_{j-1}^\T - \tW_j \tX_{j-1}^\T\|
    \\
    &\leq
    \|D_jW_j M_j \hX_{j-1}^\T - D_j W_j M_j \tX_{j-1}^\T\|
    +
    \|D_jW_j M_j \tX_{j-1}^\T - \tW_j \tX_{j-1}^\T\|
    \\
    &\leq
    \|D_jW_j M_j\|_2 \| \hX_{j-1}^\T - \tX_{j-1}^\T\|_\tF
    +
    \|D_j W_j M_j - \tW_j\|_2\| \tX_{j-1}^\T\|_\tF
    \\
    &\leq
    \sqrt{m} \|W_j\|_\tF
    \sbr[2]{ 
      \|X\|_\tF m^{(j-1)/2} \sum_{i <j} \eps_i \prod_{\substack{l < j\\l \neq i}}\|W_j\|_\tF
    }
    +
    \eps_j \|X\|_\tF \prod_{i < j} \|W_j\|_2
    \\
    &\leq
    \|X\|_\tF m^{j/2} \sum_{i\leq j} \eps_i \prod_{\substack{l \leq j\\l \neq i}}\|W_j\|_\tF,
  \end{align*}
  which establishes the desired bound on the error.

  The next step is to optimize $k_j$.
  Let $\eps>0$ be arbitrary,
  and set $\eps_j^{-1} := \eps^{-1} 2L\sqrt{m} \|X\|_\tF \prod_{i\neq j} \|W_i\|_\tF$,
  whereby
  \[
    \|\hX_L^\T - \tX_L^\T\|_\tF
    \leq \frac{\eps}{2},
    \qquad
    |\cC_j|
    \leq
    m^{k_j+k_{j-1}} \del{ \frac {4m L \sqrt{k_j}\|X\|_\tF \prod_i \|W_i\|_\tF}{\eps} }^{k_jk_{j-1}}.
  \]
  The overall network cover $\cN$ is the product of the covers for all layers,
  and thus has cardinality satisfying
  \begin{align*}
    \ln |\cN|
    &\leq
    \sum_j \ln |\cC_j|
    \leq
    2 \sum_j k_j\ln m
    +
    \sum_j k_j k_{j-1}
    \ln
    \del{ \frac {4m L \sqrt{k_j}\|X\|_\tF \prod_i \|W_i\|_\tF}{\eps_j} }
    \\
    &\leq
    2 \sum_j k_j\ln m
    +
    \sum_j 2k_j^2
    \ln
    \del{ \frac {4m L \sqrt{k_j}\|X\|_\tF }{\eps_j} }
    +
    \sbr{\sum_j 2k_j^2 }
    \cdot
    \sbr{\sum_j \ln \|W_j\|_\tF }
    .
  \end{align*}
  To choose $(k_1,\ldots,k_L)$, letting $X_L^\T$ denote the output of the original
  unsparsified network, note firstly that the full error bound satisfies
  \begin{align*}
    \|X_L^\T - \tX_L^\T\|
    &\leq
    \|X_L^\T - \hX_L^\T\| + \|\hX_L^\T -  \tX_L^\T\|
    \\
    &\leq
    \sum_i \frac {\alpha_i}{\sqrt{k_i}}
    +\frac {\eps}{2},
    &\textup{where }
    \alpha_i := \|X\|_\tF \sbr{ \prod_i \|W_i\|_2} \sqrt{\sr(W_i)}.
  \end{align*}
  To choose $k_i$, the approach here is to minimize a Lagrangian corresponding
  to the cover cardinality, subject to the total cover error being $\eps$.
  Simplifying the previous expressions and noting $2k_jk_{j-1} \leq k_j^2 + k_{j-1}^2$,
  whereby the dominant term in $\ln|\cN|$ is $\sum_j k_j^2$,
  consider Lagrangian
  \[
    L(k_1,\ldots,k_l, \lambda)
    := \sum_i k_i^2 + \lambda \del{\sum_i \frac {\alpha_i}{\sqrt{k_i}} - \frac \eps 2},
  \]
  which has critical points when each $k_i$ satisfies
  \[
    \frac{k_i^{5/2}}{\alpha_i} = \frac {\lambda}{4},
  \]
  thus $k_i := \alpha_i^{2/5}/Z$ with $Z:= \eps^2 / (2\sum_j \alpha_j^{4/5})^2$.
  As a sanity check (since it was baked into the Lagrangian),
  plugging this into the cover error indeed gives
  \[
    \|X_L^\T - \tX_L^\T\|
    \leq
    \sum_i \frac {\alpha_i}{\sqrt{k_i}} + \frac \eps 2
    =
    \sqrt{Z} \sum_i \alpha_i^{4/5} + \frac \eps 2
    =
    \eps.
  \]
  To upper bound the cover cardinality, first note that
  \[
    \sum_i k_i^2
    = \frac 1 {Z^2} \sum_i \alpha_i^{4/5}
    = \frac 4 {\eps^4} \del[2]{ \sum_i \alpha_i^{4/5} }^5,
  \]
  whereby
  \begin{align*}
    \ln|\cN|
    &= \tcO\del{
      \sbr[2]{\sum_i k_i^2}
      \cdot
      \sbr[2]{\sum_i \ln \|W_i\|_\tF}
    }
    \\
    &= \frac \beta {\eps^4}
    &\hspace{-6em}\textup{where } \beta =
    \tcO\del[3]{
      \|X\|_\tF^4 \sbr[2]{\prod_j \|W_j\|_2^4}\sbr[2]{\sum_i\sr(W_i)^{2/5}}^5
    \sbr[2]{\sum_i \ln \|W_i\|_\tF}}.
  \end{align*}
  
  The final step is to apply a Dudley entropy integral \citep{shai_shai_book},
  which gives
  \begin{align*}
    n\Rad(\cF)
    = \inf_\zeta\del{ 4\zeta\sqrt{n} + 12 \int_{\zeta}^{\sqrt n} \frac {\sqrt{\beta}}{\eps^2}\dif\eps}
      = \inf_\zeta\del{ 4\zeta\sqrt{n} + 12 \sbr{ \frac 1 \zeta - \frac 1 {\sqrt n} } \sqrt{\beta}}.
  \end{align*}
  Dropping the negative term gives an expression of the form
  $a\zeta + b/\zeta$, which is convex in $\zeta>0$
  and has critical point at $\zeta^2 = b/a$,
  which after plugging back in gives an upper bound $2\sqrt{ab}$,
  meaning
  \[
    n\Rad(\cF)
    \leq
    2\del[2]{4\sqrt{n} \cdot 12\sqrt{\beta} }^{1/2}
    = 8\sqrt{3} n^{1/4} \beta^{1/4}.
  \]
  Dividing by $n$ and expanding the definition of $\beta$
  gives the final Rademacher complexity bound.
\end{proof}

\end{document}